\documentclass[twoside,11pt]{article}

\usepackage{marginnote}
\usepackage{jmlr2e}
\usepackage{times,subfigure}
\usepackage{epsfig}
\usepackage{graphicx}
\usepackage{amsmath}
\usepackage{amssymb,bm}
\usepackage{amsfonts}
\usepackage{booktabs}
\usepackage{threeparttable}
\usepackage{appendix}
\usepackage[ruled,linesnumbered]{algorithm2e}

\usepackage{color}
\usepackage{lineno}
\usepackage{multirow}
\usepackage{makecell}
\usepackage{microtype}
\usepackage{url}            
\usepackage{nicefrac}       
\usepackage{mathrsfs}
\usepackage{pgflibraryarrows}
\usepackage{pgflibrarysnakes}
\usepackage{tikz}
\usepackage{pgfplots}

\DeclareMathOperator*{\argmin}{argmin}

\usepackage{lastpage}
\jmlrheading{22}{2021}{1-\pageref{LastPage}}{6/19; Revised
	3/21}{6/21}{19-482}{Fanghui Liu, Lei Shi, Xiaolin Huang, Jie Yang, and Johan A.K. Suykens}
\ShortHeadings{title}{Liu, Shi, Huang, Yang, and Suykens}

\begin{document}

	\pagenumbering{arabic}
	\setcounter{page}{1}
	
	
	
	\ShortHeadings{Generalization Properties of hyper-RKHS and its Applications}{Liu, Shi, Huang, Yang and Suykens}
	\firstpageno{1}

\title{Generalization Properties of hyper-RKHS and its Applications}

\author{\name Fanghui Liu\thanks{Fanghui Liu and Lei Shi contributed equally to this work. Corresponding authors: Fanghui Liu and Jie Yang.} \email fanghui.liu@kuleuven.be \\
\addr Department of Electrical Engineering, ESAT-STADIUS, KU Leuven \\
Kasteelpark Arenberg 10, Leuven, B-3001, Belgium
       \AND
       \name Lei Shi${\color{black}^*}$ \email leishi@fudan.edu.cn \\
       \addr Shanghai Key Laboratory for Contemporary Applied Mathematics\\
       School of Mathematical Sciences, Fudan University \\
       Shanghai, 200433, China
   \AND
\name Xiaolin Huang \email xiaolinhuang@sjtu.edu.cn
\AND
\name Jie Yang \email jieyang@sjtu.edu.cn \\
\addr Institute of Image Processing and Pattern Recognition,
Shanghai Jiao Tong University \\
\addr Institute of Medical Robotics, Shanghai Jiao Tong University\\
Shanghai, 200240, China
\AND
\name Johan A.K. Suykens \email johan.suykens@esat.kuleuven.be \\
\addr Department of Electrical Engineering, ESAT-STADIUS, KU Leuven \\
Kasteelpark Arenberg 10, Leuven, B-3001, Belgium
}

\editor{Jean-Philippe Vert}

\maketitle

\begin{abstract}
This paper generalizes regularized regression problems in a hyper-reproducing kernel Hilbert space (hyper-RKHS), illustrates its utility for kernel learning and out-of-sample extensions, and proves asymptotic convergence results for the introduced regression models in an approximation theory view.
Algorithmically, we consider two regularized regression models with bivariate forms in this space, including kernel ridge regression (KRR) and support vector regression (SVR) endowed with hyper-RKHS, and further combine divide-and-conquer with Nystr\"{o}m approximation for scalability in large sample cases. This framework is general: the underlying kernel is learned from a broad class, and can be positive definite or not, which adapts to various requirements in kernel learning. Theoretically, we study the convergence behavior of regularized regression algorithms in hyper-RKHS and derive the learning rates, which goes beyond the classical analysis on RKHS due to the non-trivial independence of pairwise samples and the characterisation of hyper-RKHS. Experimentally, results on several benchmarks suggest that the employed framework is able to learn a general kernel function form an arbitrary similarity matrix, and thus achieves a satisfactory performance on classification tasks.
\end{abstract}

\begin{keywords}
  hyper-RKHS, approximation theory, kernel learning, out-of-sample extensions
\end{keywords}

\section{Introduction}

Reproducing kernel Hilbert spaces (RKHS) \citep{aronszajn1950theory,saitoh2016theory} provide the ability to approximate functions by nonparametric functional representations, and thus have developed into an important tool in many areas, especially kernel methods in machine learning \citep{suykens2002least,scholkopf2018learning}.
For any two data points $\bm x, \bm x' \in X$, kernel methods work under the setting: the original data are mapped to high or infinite dimensional RKHS such that $k(\bm x, \bm x') = \langle \varphi(\bm x), \varphi(\bm x') \rangle_{\mathcal{H}}$ with an implicit feature mapping $\varphi: X \rightarrow \mathcal{H}$.
Here the kernel is required to be symmetric and positive definite (PD)\,\footnote{Because of the confusing terminology on functions and their counterparts on matrices, here, we follow the convention that a positive definite function corresponds to a positive semi-definite (PSD) matrix.}, and corresponds a unique RKHS. The ``reproducing" terminology indicates the \emph{reproducing property} of RKHS
\begin{equation*}
	f(\bm x) = \langle f, k(\bm x, \cdot) \rangle_{\mathcal{H}},~~\mbox{for all $f \in \mathcal{H}$ and $\bm x \in X$ }\,.
\end{equation*}
Accordingly, for each $\bm x \in X$, the point evaluation function $f \rightarrow f(\bm x)$ is continuous, showing that strong convergence in RKHS implies point-wise convergence.
This property makes RKHS an appealing choice in machine learning problems with nice theoretical guarantees in an approximation theory view \citep{caponnetto2007optimal,cucker2007learning}.
The structure of RKHS is determined by the choice of the kernel $k$, but selecting appropriate kernels is not a trivial task.
In fact, RKHS is not large enough \citep{bach2017breaking,steinwart2020reproducing} and thus would lead to a lack of adaptivity in many learning problems.
In this paper, we consider to learn the kernel from a hyper-reproducing kernel Hilbert space (hyper-RKHS) \citep{Cheng2005Learning} associated with the reproducing hyper-kernel \citep{kondor2007gaussian}.
Different from RKHS, every element in hyper-RKHS is a kernel function, which allows for significant model flexibility from a broad class.
Specifically, the learned kernel endowed by hyper-RKHS has the property of translation
and rotation invariant simultaneously \citep{motai2015kernel}, and thus is extensively applied to feature representations \citep{raj2016local} and other applications such as classification \citep{Tsang2006Efficient}, density estimation \citep{ganti2008hyperkernel}, and out-of-sample extensions \citep{Pan2016Out}.

Learning in hyper-RKHS $\underline{\mathcal{H}}$ is general to cover various settings or applications, e.g., kernel learning, out-of-sample extensions, and indefinite kernels (real, symmetric but not positive definite).
First, in kernel learning aspect, let $X$ be a compact metric space, $\widetilde{Y} \subseteq  \mathbb{R}$ be the output space\footnote{The symbol $Y$ is used to denote another output space that will be introduced later.}, $\mathcal{D}=\{ (\bm x_i, \tilde{y}_i)\}_{i=1}^m$ be the training set with $\bm x_i \in X$ and the response $\tilde{y}_i \in \widetilde{Y}$.
Let $k(\cdot,\cdot): X \times X \rightarrow \mathbb{R}$ be a positive definite kernel function that we need to learn.
Figure~\ref{framework} shows the kernel learning framework in hyper-RKHS via two stages.
Stage 1 (in blue) formulates kernel learning as a regression problem in hyper-RKHS by minimizing the quality function $\mathcal{T}(k, \bm K)$ between the learned kernel $k$ and the \emph{target kernel} $\tilde{\bm y} \tilde{\bm y}^{\!\top}$. Here the used \emph{target kernel}, regarded as an ``ideal kernel", is able to guide the kernel learning task in hyper-RKHS as it can directly recognise the training data with certainly 100\% accuracy. The used quality function $\mathcal{T}(k, \bm K)$ evaluates the similarity between the learned kernel $k$ and the pre-given $\tilde{\bm y} \tilde{\bm y}^{\!\top}$, which will be formally defined in Section~\ref{sec:hyper}.
This scheme is similar to \emph{target alignment} \citep{Cortes2012Algorithms,Wang2015An} that evaluates how well the learned Gram matrix aligns to the \emph{target kernel} based on the multiple kernel learning framework. However, different from them, the studied kernel learning framework here is formulated as a regularized regression problem in hyper-RKHS from a broader class instead of only acquiring the linear combination of basic kernel(s).
In stage 2 (in red), we aim to find a hypothesis function $f \in \mathcal{H}$ evaluated by a convex continuous loss functional $\mathbb{E}_{\bm x, \tilde{y}}[\ell (f(\bm x), \tilde{y})]$ and a Tikhonov regularizer $\| f \|^2_{\mathcal{H}}$, where the convex loss $\ell: \mathcal{H} \times \widetilde{Y} \rightarrow \mathbb{R}$ quantifies the merit of the evaluation $f(\bm x)$ at $\bm x \in X$.
Note that the learned kernel can be indefinite that is not associated with RKHS, we then discuss it later in this section.

Second, if we consider other types beyond the \emph{target kernel}, e.g., a pre-given kernel matrix $\bm K$, the above kernel learning process is transformed to tackle the out-of-sample extensions problem \citep{Bengio2004Out,Pan2016Out}, i.e., learning an underlying/unknown kernel $k$ from a pre-defined or manually specified kernel/similarity matrix $\bm K$. In fact, the above kernel learning process by the \emph{target kernel} can be also regarded as a special case of this framework due to the \emph{target kernel} only defined on the training data. 
The out-of-sample extension topic widely exists in many research areas, such as 1) nonparametric kernel learning~\citep{Lu2009Geometry,liu2020learning}: the kernel learning scheme is in a data specific manner, i.e., obtains the ``similarity values" instead of learning a similarity function.
2) metric learning~\citep{Kulis2013Metric,Jain2017Learning}: it often learns a Mahalanobis-like matrix from the given training data, but is infeasible to new data.
3) nonlinear manifold learning~\citep{Hong2013Anml,Xie2013Onn}: the low-dimensional data coordinates are computed only for the initially available training data and can not be extended to the test data in a straightforward way.
In these cases, since the learned Mahalanobis-like matrix can be regarded as a kernel matrix, and the nonlinear mapping in manifold learning can be represented by a kernel, here we investigate them in a unified framework.
Hence, we aim to tackle the following question:
\centerline{\emph{How to learn an underlying kernel/similarity function from the pre-given data-specific matrix?}}

  \begin{figure}[t]
        \centering
          \begin{tikzpicture}[scale=0.82]
          \colorlet{blueg}{blue!70} \colorlet{redg}{red!70}
      \path (0,2) node[circle,draw,fill=lightgray] (x1) {\large{\color{red}$f^*$}}
      (9.5,2) node[circle,draw,fill=lightgray](y1) {\Large{\color{blue}$\bm k^*$}};
 \draw[line width=2pt,redg,latex-] (x1) -- node[auto] {} (y1);
      \path (17.5,2)node[circle,draw,fill=lightgray] (y2){$\tilde{\bm y}{\tilde{\bm y}}^{\!\top}$};
      \draw[line width=2pt,blueg,latex-] (y1) -- node[auto] {} (y2);
      \path (9,4) node(a){~};  \path (11,4) node(b){~};
      \draw[line width=2pt,redg,latex-] (a) -- node[auto] {} (b);
       \path (14.4,4) node(c){{\color{red}stage 2: learning in RKHS or RKKS}};
             \path (9,4.5) node(a1){~};  \path (11,4.5) node(b1){~};
      \draw[line width=2pt, blueg,latex-] (a1) -- node[auto] {} (b1);
       \path (14,4.5) node(c1){{\color{blue}stage 1: learning in hyper-RKHS}};
   \path (5,2.4) node(c1){{\color{red}classifier: SVM in RKHS or RKKS}};
   \path (13.4,2.5) node(c3){{\color{blue}SVR or KRR in hyper-RKHS}};
   \path (13.7,1.2) node(q1){$k^* \!:=\! \argmin\limits_{k \in \underline{\mathcal{H}}}  \mathcal{T}\big(k, \tilde{\bm y}{\tilde{\bm y}}^{\!\top} \!\big) \!+\! \lambda \langle k, k \rangle_{\underline{\mathcal{H}}}$};
    \path (5.2,1.2) node(q1){$f^* := \argmin\limits_{{ f} \in \mathcal{H}(k^*)}  \mathbb{E}_{\bm x, \tilde{y}} \big[\ell (f(\bm x), \tilde{y}) \big] + \lambda \| f \|_{\mathcal{H}(k^*)}^2$};
      \end{tikzpicture}
        \caption{The two-stage kernel learning framework endowed by hyper-kernels with stage 1 (in blue): learning the kernel $k$ by minimizing the quality functional $\mathcal{T}$ between $k$ and the target kernel $\tilde{\bm y} \tilde{\bm y}^{\!\top}$ in hyper-RKHS and stage 2 (in red): learning the hypothesis $f$ in RKHS or RKKS for classification.}
\label{framework}
\vspace{-0.2cm}
\end{figure}

Third, the learned kernel in stage 1 is not limited to be positive definite since hyper-RKHS has the capability of generating an indefinite kernel that is associated with a reproducing kernel Kre\u{\i}n space (RKKS) \citep{Cheng2004Learning,bognar1974indefinite} instead of RKHS. 
This operation is reasonable since we can hardly predict whether the underlying kernel is positive definite or indefinite even if the pre-defined $\bm K$ or the \emph{target kernel} is PSD in the above two situations.
Additionally imposing the positive definiteness on the learned kernel would exclude indefinite kernel learning \citep{Ga2016Learning,Schleif2015Indefinite}.
In practice, indefinite kernel learning is ubiquitous in many real-world applications, e.g., the hyperbolic tangent kernel \citep{smola2001regularization} and the ``dot-product attention” in Transformers \citep{wright2021transformers}. Besides, some PD kernels would degenerate to indefinite ones, e.g., a linear combination of PD kernels (with negative coefficient) \citep{Cheng2005Learning}, dot-product kernels by $\ell_2$ normalization \citep{pennington2015spherical,liu2020fast}, and Gaussian kernels with some geodesic distances \citep{Feragen2015Geodesic}.
Regarding to descriptions about RKKS, and justification, model formulation, optimization for indefinite kernel based algorithms, see \citep{Schleif2015Indefinite,oglic18a,liu2020analysis} among others.
Accordingly, Firgure~\ref{framework} include indefinite kernel learning to cover various requirements for a general kernel learning framework endowed by hyper-RKHS.

Now that learning in hyper-RKHS is adopted for numerous research fields, there is a key question left unanswered in the theoretical aspect.
The convergence behavior of learning algorithms in $\underline{\mathcal{H}}$ has not been fully investigated in learning theory.
In this paper, we generalize two regularized regression problems in hyper-RKHS, illustrates its utility for kernel learning and out-of-sample extensions, and proves asymptotic convergence results for the introduced regression models in an approximation theory. In particular, we make the following contributions:

Algorithmically, in Section~\ref{sec:hyper}, motivated by \cite{Cheng2005Learning}, we consider regularized regression problems with squared loss and $\varepsilon$-insensitive loss (i.e., KRR and SVR) in hpyer-RKHS for kernel learning and out-of-sample extensions.
Specifically, the developed models are general to output PD or indefinite kernels, which allows for significant model flexibility and universality.
To make our kernel learning framework applicable to large scale situations, we combine the divide-and-conquer scheme with Nystr\"{o}m approximation for further improvement on computational efficiency.

Theoretically, our main results on generalization properties of KRR and SVR in hyper-RKHS are presented in Section~\ref{sec:learningrate}, and the proofs are given in Section~\ref{sec:prooframes}.
Since learning in hyper-RKHS involves pairs of samples, which is no longer mutual pairwise independent \citep{Luby2005Pairwise}, the standard approximation analysis for RKHS \citep{cucker2007learning,Suzuki2012Fast,Mendelson2010REGULARIZATION} in learning theory cannot be directly applied to hyper-RKHS. 
This work addresses this issue, provides the asymptotic analysis of regularized regression problems in hyper-RKHS, and fills a theoretical gap.

Experimentally, in Section~\ref{experiment}\,, we present numerical results on several benchmark datasets to verify the effectiveness of our two-stage kernel learning framework.
For stage 1, we observe that our regression methods in hyper-RKHS accurately fits the given kernel matrix including PSD and non-PSD ones with small approximation errors.
For stage 2, the learned kernel incorporated into SVM performs well in terms of classification accuracy whatever the pre-given kernel matrix is.
Further, the developed kernel scalability method reduces the complexity of our kernel learning algorithms by orders of magnitude.
Finally we discuss the related work close to our framework in Section~\ref{sec:diss} and draw the conclusion in Section~\ref{sec:con}\,.

\section{Learning in hyper-RKHS}
\label{sec:hyper}
In this section, we formulate the regression problem in hyper-RKHS as a regularized risk minimization problem, and then devise two regression algorithms associated with hyper-RKHS.

\subsection{Regularized Risk Minimization in hyper-RKHS}
The elements in hyper-RKHS are kernel functions, and thus the associated reproducing kernel is called the hyper-kernel (kernel of kernel), termed as $\underline{k}$.
The definition of this space and its associated (reproducing) hyper-kernel is presented as follows.

\begin{definition}\label{definitehyperrkhs}
[hyper-RKHS and its (reproducing) hyper-kernel \citep{Cheng2005Learning}] Let $X$ be a compact metric space, $\underline{{X}}={X} \times {X}$ and $\underline{\mathcal{H}}$ denotes a Hilbert space of functions $k:~\underline{{X}} \rightarrow \mathbb{R}$.
Then for any $\underline{\bm{x}}, \underline{\bm{x}}' \in \underline{{X}}$, the inner product space $\underline{\mathcal{H}}$ is called a hyper-RKHS endowed with the dot product $\langle \cdot, \cdot \rangle_\mathcal{\underline{H}}$ (and the norm $\| k \|_{\mathcal{\underline{H}}} = \sqrt{\langle k, k \rangle_\mathcal{\underline{H}}}$) if there exists a hyper-kernel $\underline{k}: \underline{{X}} \times \underline{{X}} \rightarrow \mathbb{R}$ with the following properties:
\begin{itemize}
	  \item (reproducing) $\underline{k}$ has the reproducing property $\langle k, \underline{k}(\underline{\bm{x}},\cdot) \rangle_{\underline{\mathcal{H}}} = k(\bm{\underline{x}})$ for all $k \in \underline{\mathcal{H}}$; in particular, we have $\langle \underline{k}(\underline{\bm{x}},\cdot), \underline{k}(\underline{\bm{x}}',\cdot) \rangle_{\underline{\mathcal{H}}} = \underline{k}(\underline{\bm{x}},\underline{\bm{x}}')$.
	\item (symmetric) $\underline{k}((\bm{x},\bm{y}),(\bm{r},\bm{s})) = \underline{k}((\bm{y},\bm{x}),(\bm{r},\bm{s})) $ for all $\bm{x}, \bm{y}, \bm{r}, \bm{s} \in {X}$. Further, for any fixed $\bm{\underline{x}} \in \underline{X}$, the hyper kernel $\underline{k}$ is a kernel in its second argument, i.e., $k(\bm{x},\bm{x}') := \underline{k}(\bm{\underline{x}}, (\bm{x},\bm{x}'))$ with $\bm{x}, \bm{x}' \in {X}$.
	\item (positive definite) $\underline{k}$ is positive definite on $\underline{X}$ and $\underline{k}(\bm{\underline{x}},\cdot)$ is positive definite on $X$ for any $\underline{\bm{x}} \in \underline{{X}}$.
  \item $\underline{k}$ spans $\mathcal{\underline{H}}$, i.e., $\mathcal{\underline{H}}=\overline{\text{span}\{ \underline{k}(\bm{\underline{x}},\cdot)| \bm{\underline{x}} \in \mathcal{\underline{X}}\}}$.
\end{itemize}
\end{definition}
Here we can view the hyper-kernel as a function of four arguments, $\underline{k}((\bm x_1, \bm x_2), (\bm x'_1, \bm x'_2))$ or a function of two pairs, $\underline{k}(\underline{\bm x}, \underline{\bm x'})$ with $\underline{\bm x} := (\bm x_1, \bm x_2)$ and $\underline{\bm x'} := (\bm x'_1, \bm x'_2))$.
The reproducing and symmetric property ensures $\underline{k}$ to be a kernel, and $\underline{k}(\underline{\bm x}, (\bm x, \bm x'))$ is also a kernel for any fixed pair $\underline{\bm x}$.
Besides, $\underline{k}$ should be positive definite so as to induce a hyper-RKHS (a special case of RKHS) based on Definition~\ref{definitehyperrkhs}.
Denote $C({\underline{X}})$ as the space of continuous functions on ${\underline{X}}$ with the norm $\| f \|_{\infty}:=\sup_{\underline{\bm x} \in \underline{X}}|f(\underline{\bm x})|$ for $f\in C({\underline{X}})$. Due to the continuity of the kernel function $\underline{k}$ and compactness of ${\underline{X}}$, we have
\begin{equation*}
  \mathcal{G} := \sup_{\underline{\bm x} \in \underline{X}}  \sqrt{\underline{k}( \underline{\bm x},\underline{\bm x})} < \infty \,.
\end{equation*}
Hence the reproducing property in hyper-RKHS indicates that $\underline{\mathcal{H}}\subset C({\underline{X}})$ and
\begin{equation}\label{fnormdiff}
  \|k\|_{\infty} = \sup_{\underline{\bm x} \in \underline{X}} \big| \langle k, \underline{k}(\underline{\bm x},\cdot) \rangle_{\underline{\mathcal{H}}} \big| \leq \mathcal{G}  \langle k,k\rangle_{\mathcal{\underline{H}}}  = \mathcal{G} \| k\|_{\underline{\mathcal{H}}}^2, \quad \forall k \in \mathcal{\underline{H}}\,.
\end{equation}

We can see that $\mathcal{\underline{H}}$ is different from a normal RKHS $\mathcal{{H}}$ on the particular form of its index set ${\underline{X}}$ and the additional condition on the hyper-kernel $\underline{k}$ to be symmetric in its first two arguments, and thus in its second two arguments as well.
Here we investigate the regularized regression problem in hyper-RKHS, which is formulated as
\begin{equation}\label{mainQ}
  \min \limits_{k \in \underline{\mathcal{H}}} ~ \frac{1}{m^2} \sum_{i,j=1}^{m} \mathcal{T}(Y_{ij}, k(\bm x_i, \bm x_j)) + \lambda \langle k, k \rangle_{\underline{\mathcal{H}}}\,,
\end{equation}
where the first term is the quality functional $\mathcal{T}(\bm Y, k)$ based on its point-wise definition and the regularization parameter $\lambda := \lambda(m) > 0$ satisfies $\lim_{m \rightarrow \infty} \lambda(m) = 0$.
The response variable is $\bm Y$.
In kernel learning via target alignment, $\bm Y$ is chosen as the target kernel (matrix), i.e., $\bm Y = \tilde{\bm y} \tilde{\bm y}^{\!\top}$.
For the out-of-sample extensions issue, $\bm Y$ is chosen as a pre-given kernel/similarity matrix $\bm K$.
The quality functional $\mathcal{T}(\bm Y, k)$ focuses on the approximation ability of a kernel function $k$ to the given $\bm Y$.
In regression problem, it should satisfy $\mathcal{T}(Y_{ij}, k(\bm x_i, \bm x_j))=\ell(Y_{ij}-k(\bm x_i, \bm x_j))$, where the loss function $\ell(\cdot)$ can be chosen as the squared loss in least-squares, the $\varepsilon$-insensitive loss function in SVR, and so on.
Using the representer theorem in hyper-RKHS \citep{Cheng2005Learning}, the minimizer $k^* \in \underline{\mathcal{H}}$ of problem~\eqref{mainQ} admits
\begin{equation}\label{rkrkks}
\begin{split}
   k^*(\bm x, \bm x') = & \sum_{i=1}^m \sum_{j=1}^{m} \beta_{ij} \underline{k}\big( (\bm x_i, \bm x_j), (\bm x, \bm x') \big) \quad
  \mbox{with}~\bm x, \bm x' \in X,~ \beta_{ij} \in \mathbb{R}\,,
  \end{split}
\end{equation}
where $\bm \beta$ is the expansion coefficient matrix.
In our formulation, $k^*$ can be a general kernel, \emph{i.e.}, PD or indefinite.
To be exact, in hyper-RKHS, the hyper-kernel $\underline{k}$ is positive definite, but the coefficient $\beta_{ij}$ in the above formulation might be negative, which results in an indefinite kernel endowed by RKKS.
Therefore, as we expect, the learned solution $k^*$ can be a positive definite kernel or an indefinite one.
Such a general framework in hyper-RKHS provides strong adaptivity in kernel learning.

\subsection{Regression Models in hyper-RKHS}
Here we consider two regression algorithms including KRR and SVR in hyper-RKHS.
By choosing the squared loss, the least-squares regression algorithm in hyper-RKHS is
\begin{equation}\label{krr}
  \min_{k \in \underline{\mathcal{H}}}~ \frac{1}{m^2} \sum_{i,j=1}^{m} \Big( k(\bm x_i, \bm x_j) - Y_{ij} \Big)^2 + \lambda \langle k, k \rangle_{\underline{\mathcal{H}}}\,,
\end{equation}
where $\lambda$ seeks for a tradeoff between the complexity of $k$ and the fitting ability in regression.
Compared to the conventional kernel ridge regression problem, our formulation in Eq.~\eqref{krr} is in a bivariate form because we optimize over the kernel function.

Using the representer theorem in hyper-RKHS, problem~\eqref{krr} can be reformulated as
\begin{equation}\label{krreq}
  \min_{\bm \beta} \frac{1}{m^2} \big\| \underline{\bm K} \mathop{\mathrm{vec}}({\bm \beta}) - \mathop{\mathrm{vec}}(\bm Y) \big\|^2_2 + \lambda \mathop{\mathrm{vec}}({\bm \beta})^{\top} \underline{\bm K} \mathop{\mathrm{vec}}({\bm \beta})\,,
\end{equation}
with the coefficient vector $\mathop{\mathrm{vec}}({\bm \beta}) \in \mathbb{R}^{m^2}$.
The hyper-kernel matrix is $\bm{\underline{K}} \in \mathbb{R}^{m^2 \times m^2}$ with entries $\bm{\underline K}_{d(i,j,m)d(r,s,m)} = \underline{k}\big( (\bm{x}_i, \bm{x}_j), (\bm{x}_r, \bm{x}_s) \big)$, in which the function $d(i,j,m)=m(i-1)+j$ maps the pair $(i,j)$ to the row or column index of $\bm{\underline K}$.
The hyper-kernel matrix $\underline{\bm K}$ is PSD when we choose a positive definite hyper-kernel $k$.
This model is studied in \citep{Pan2016Out} by additionally adding the non-negative constraint on $\bm \beta$ so as to output a PD kernel.
Comparably, the expansion coefficient $\beta_{ij}$ is not constrained to be nonnegative, which breaks through the restriction of the nonnegative constraint on the expansion coefficients in the representer theorem \eqref{rkrkks} in hyper-RKHS and thus is able to yield an indefinite kernel $k$.
Accordingly, the solution to problem~\eqref{krreq} can be directly given by
\begin{equation}\label{krrsolu}
  \mathop{\mathrm{vec}} (\bm \beta) = \Big(\underline{\bm K} + \lambda m^2 {\bm I}\Big)^{-1} \mathop{\mathrm{vec}}(\bm Y) \,.
\end{equation}
It can be noticed that, solving this model would be time-consuming due to the hyper-kernel matrix $ \underline{\bm K} \in \mathbb{R}^{m^2 \times m^2}$.
In Section~\ref{sec:large}, we will consider its scalability in large scale datasets by the combination of distributed learning and Nystr\"{o}m approximation.

Apart from exploiting the squared loss in hyper-RKHS, we study the $\varepsilon$-insensitive loss for bivariate-support vector regression as a quality functional for regression, namely
\begin{equation}\label{svrprimal}
  \begin{split}
&\mathop{\mathrm{min}}\limits_{k \in \underline{\mathcal{H}}, b, \hat{\bm \xi}, \check{\bm \xi} }~~ \frac{1}{2} \langle k, k \rangle_{\underline{\mathcal{H}}} + C \sum_{i,j=1}^{m} ( \hat{\xi}_{ij} + \check{\xi}_{ij} ) \\
&s.t.~k(\bm{x}_i, \bm{x}_j) +b - Y_{ij} \leq \varepsilon + \hat{\xi}_{ij} \\
&~~~~~Y_{ij} - k(\bm{x}_i, \bm{x}_j) - b \leq \varepsilon + \check{\xi}_{ij} \\
&~~~~~ \hat{\xi}_{ij}, ~\check{\xi}_{ij} \geq 0 \quad \forall i,j = 1,2,\cdots,m\,,
\end{split}
\end{equation}
where $b$ is a bias term, $C$ is a tradeoff between the fitting ability and the smoothness of the learned $k$.
The notations $\hat{\bm \xi}, \check{\bm \xi}$ are two slack variables associated with the quality functional $\mathcal{T}\mathcal(\bm Y,k)$.
Analogous to the derived KRR in hyper-RKHS, our SVR formulation is also in a bivariate form.
By the representer theorem in hyper-RKHS, the dual form of problem~\eqref{svrprimal} is formulated as
 \begin{equation*}\label{hrkkssvrdual}
 \begin{split}
 &\mathop{\mathrm{max}}\limits_{\hat{\bm \beta}, \check{\bm \beta} } -\frac{1}{2} \sum_{i,j,r,s=1}^{m} \!\!\!\! (\hat{\beta}_{ij} \!- \! \check{ \beta}_{ij}) (\hat{\beta}_{rs} \!-\! \check{\beta}_{rs}) \underline{k}\big( (\bm{x}_i, \bm{x}_j), (\bm{x}_r, \bm{x}_s) \big) +\!\!\! \sum_{i,j=1}^{m}\! \!\Big\{ \! Y_{ij}(\hat{ \beta}_{ij} \!-\! \check{ \beta}_{ij})\! -\! \varepsilon(\hat{\beta}_{ij} \! +\! \check{ \beta}_{ij})\! \Big\}\\
&s.t.~0 \leq \hat{ \beta}_{ij}, \check{\beta}_{ij} \leq C, ~\sum_{i,j=1}^{m} (\hat{\beta}_{ij} - \check{ \beta}_{ij}) = 0\,,
 \end{split}
\end{equation*}
with the expansion coefficient $\beta_{ij}=\hat{\beta}_{ij} - \check{ \beta}_{ij}$.
We can see that the expansion coefficients $\beta_{ij} \in [-C,C]$ may be negative, which has the capability of resulting in an indefinite kernel $k$ even if we choose a positive definite hyper-kernel $\underline{k}$.
Further, the above equation can be rewritten in a compact form
\begin{equation}\label{ualmat}
  \begin{split}
&\max \limits_{\underline{\hat{\bm \beta}}, \underline{\check{\bm \beta}}} ~-\!\frac{1}{2}(\underline{\hat{\bm \beta}} \!-\! \underline{\check{\bm \beta}})^{\!\top}\! \bm{\underline{K} }(\underline{\hat{\bm \beta}} \!-\! \underline{\check{\bm \beta}}) \! +\! (\underline{\hat{\bm \beta}} \!-\! \underline{\check{\bm \beta}})\mbox{vec}(\bm Y) \!\!-\! \varepsilon (\underline{\hat{\bm \beta}} \!+\! \underline{\check{\bm \beta}})^{\!\top}\!\bm{1}\\
&s.t.~0 \leq \underline{\hat{\bm \beta}}, \underline{\check{\bm \beta}} \leq C, ~(\underline{\hat{\bm \beta}} - \underline{\check{\bm \beta}})^{\top}{\bm{1}} = 0\,,
\end{split}
\end{equation}
where $\underline{\hat{\bm \beta}} = \mathop{\mathrm{vec}}(\hat{\bm \beta}) \in \mathbb{R}^{m^2}$, $\underline{\check{\bm \beta}} = \mathop{\mathrm{vec}}(\check{\bm \beta}) \in \mathbb{R}^{m^2}$, and $\bm 1$ is an all-one vector.
One can see that the derived SVR model in hyper-RKHS shares the similar formulation with that in RKHS, and can be also solved by the SMO algorithm \citep{Platt1999Fast}.

\subsection{Kernel Approximation in Large Scale Situations}
\label{sec:large}
Regarding to optimization algorithms for problems~\eqref{krreq} and \eqref{ualmat}, our regression models can be solved by standard optimization algorithms, i.e., the matrix inversion operator for KRR in Eq.~\eqref{krrsolu} and the SMO algorithm \citep{Platt1999Fast} for SVR.
While solving these algorithms are time-consuming due to the $m^2$ variables.
Precisely, KRR in hyper-RKHS takes $\mathcal{O}(m^6)$ time complexity and requires $\mathcal{O}(m^4)$ space to store the hyper-kernel matrix.
Thankfully, we do not need to simultaneously consider all pairs, though the hyper-kernel matrix is an $m^2 \times m^2$ matrix.
Here we develop a divide-and-conquer approach with Nystr\"{o}m approximation \citep{williams2001using,hsieh2014divide,yin2020divide,lin2020optimal} to speed up our method and reduce the required storage in large scale situations.

We take KRR in hyper-RKHS as an example to illustrate our two kernel approximation schemes, i.e., dividing the training data into several partitions and conducting Nystr\"{o}m approximation on each subset. Such approximation strategy for SVR in hyper-RKHS works in the similar fashion with KRR, and each sub-problem can be efficiently solved by liblinear \citep{ho2012large}. To detail our scalable scheme, we begin with KRR in hyper-RKHS with Nystr\"{o}m approximation, and then present the divide-and-conquer strategy.
To scale KRR in hyper-RKHS to large sample situations, the Nystr\"{o}m scheme randomly selects a subset of $M$ (often $M \ll m$) training data $\{ \widetilde{\bm x}_1, \widetilde{\bm x}_2, \cdots, \widetilde{\bm x}_M \} \subset \{ \bm x_1, \bm x_2, \cdots, \bm x_m \}$, termed as landmarks or centers, to approximate the original hyper-kernel matrix. Here the used sampling strategy can be uniform or advanced ones, e.g., leverage scores based sampling \citep{alaoui2015fast}. The solution of KRR-Nystr\"{o}m in hyper-RKHS via the used pairs $\{ (\widetilde{\bm x}_i, \widetilde{\bm x}_j) \}_{i,j=1}^M$ is given by
\begin{equation*}
\begin{split}
&	\widetilde{k}_{M, \lambda} (\bm x, \bm x') = \sum_{i,j=1}^M \! \widetilde{\beta}_{ij} \underline{k} \big((\bm x, \bm x'), (\widetilde{\bm x}_i, \widetilde{\bm x}_j) \big) \\
& ~\mbox{with}~\mathop{\mathrm{vec}} (\widetilde{\bm \beta}) \!=\! \!\Big(\underline{\bm K}_{mM}^{\!\top}\underline{\bm K}_{mM} + \lambda m^2 \underline{\bm K}_{MM} \Big)^{-1} \underline{\bm K}_{mM}^{\!\top} \mathop{\mathrm{vec}}(\bm Y) \,, 
\end{split}
\end{equation*}
where $\underline{\bm K}_{mM}  \in \mathbb{R}^{m^2 \times M^2}$ is obtained from the whole hyper-kernel matrix $\underline{\bm K}$ across samples $\{ (\bm x_i, \bm x_j) \}_{i,j=1}^m$ and $\{ (\widetilde{\bm x}_i, \widetilde{\bm x}_j) \}_{i,j=1}^M$, and $\underline{\bm K}_{MM}  \in \mathbb{R}^{M^2 \times M^2}$ is constructed by $\underline{k}\big( (\widetilde{\bm x}_r, \widetilde{\bm x}_s), (\widetilde{\bm x}_i, \widetilde{\bm x}_j) \big)$ with $i,j,r,s \in \{ 1,2,\cdots, M \}$.
Accordingly, the original hyper-kernel matrix $\underline{\bm K}$ can be approximated by Nystr\"{o}m approximation
\begin{equation*}
	\underline{\bm K} \approx \underline{\bm K}_{mM} \underline{\bm K}_{MM}^{\dagger} \underline{\bm K}_{mM}^{\!\top} \,,
\end{equation*}
where $(\cdot)^{\dagger}$ denotes the pseudo-inverse.
By virtue of Nystr\"{o}m approximation, the time complexity is reduced from $\mathcal{O}(m^6)$ to $\mathcal{O}(m^2 M^4)$, and the space complexity is from $\mathcal{O}(m^4)$ to $\mathcal{O}(m^2M^2)$.

Further, the computational efficiency can be improved if we incorporate the divide-and-conquer scheme into our Nystr\"{o}m approximation framework.
We split the training data $\{ \bm x_i \}_{i=1}^m$ into $v$ disjoint subsets $\{ \mathcal{V}_1, \mathcal{V}_2,\dots, \mathcal{V}_v \}$, and assume that the sample size of each partition is the same for simplicity, i.e., $|\mathcal{V}_1| = |\mathcal{V}_2| = \cdots = |\mathcal{V}_v| = n$ such that $m = nv$.
Then the used divide-and-conquer framework generates the global solution as the average of local estimators
\begin{equation*}
\bar{k}_{M,\lambda}(\bm x, \bm x') = \frac{1}{v}\sum_{c=1}^v \widetilde{k}_{M, \mathcal{V}_c, \lambda}(\bm x, \bm x')\,,
\end{equation*}
where $\widetilde{k}_{M, \mathcal{V}_c, \lambda}(\bm x, \bm x')$ is the Nystr\"{o}m estimator  on $\mathcal{V}_c$ ($c=1,2,\dots,v$) satisfying
\begin{equation}\label{nysdckrr}
\begin{split}
& \widetilde{k}_{M, \mathcal{V}_c, \lambda}(\bm x, \bm x') = \sum_{i,j=1}^M \! \bar{\beta}_{ij} \underline{k} \big((\bm x, \bm x'), (\bm x_i, \bm x_j) \big)\\
& \mbox{with}~\mathop{\mathrm{vec}} (\bar{\bm \beta}) = \Big(\underline{\bm K}_{nM}^{\!\top}\underline{\bm K}_{nM} + \lambda n^2 \underline{\bm K}_{MM} \Big)^{-1} \underline{\bm K}_{nM}^{\!\top} \mathop{\mathrm{vec}}(\bm Y_{nn}) \,, 
\end{split}
\end{equation}
where the Nystr\"{o}m landmarks $\{ \widetilde{\bm x}_1, \widetilde{\bm x}_2, \cdots, \widetilde{\bm x}_M \}$ are from $\mathcal{V}_c$ satisfying $M \leq |\mathcal{V}_c| = n$.
The matrix $\underline{\bm K}_{nM}  \in \mathbb{R}^{n^2 \times M^2}$ is obtained from the sub-hyper-kernel matrix $\underline{\bm K}^{(c,c)} \in \mathbb{R}^{n^2 \times n^2}$ corresponding to the $c$-th partition $\mathcal{V}_c$. 
The matrix $\underline{\bm K}_{MM} \in \mathbb{R}^{M^2 \times M^2}$ corresponds to the subsampling data $\underline{\bm K}^{(c,c)}$ across  $\{ (\widetilde{\bm x}_i, \widetilde{\bm x}_j) \}_{i,j=1}^M$ from $\mathcal{V}_c$.
The matrix $\bm Y_{nn}$ derives from the response matrix $\bm Y$ on $\mathcal{V}_c$ with $n$ training data.
Under this setting, the time and space complexity are further reduced to $\mathcal{O}(m^2 M^4 /v)$ and $\mathcal{O}(m^2 M^2 /v)$, respectively. The detailed process of the approximation algorithm for KRR in hyper-RKHS is
summarized in Algorithm~\ref{agolargescale}.

\begin{algorithm}[t]
	\caption{Divide-and-conquer with Nystr\"{o}m approximation for KRR in hyper-RKHS}
	\label{agolargescale}
	\KwIn{Data points $\{\bm x_i \}_{i=1}^m$, the response matrix $\bm{Y}$, the (Gaussian/Wishart) hyper-kernel matrix $\underline{\bm{K}}$, and regularized paramter $\lambda$, the number of partitions $v$, and the number of Nystr\"{o}m centers $M \leq m/v$}
	\KwOut{the final estimator $\bar{k}_{M,\lambda}(\bm x, \bm x')$}
	randomly partition the data points into $v$ disjoint subsets $\{ \mathcal{V}_t \}_{t=1}^v$.\\
	//{\tt in parallel: handle $\mathcal{V}_t$ with a local processor}\\
	randomly select $M$ data points from $\mathcal{V}_t$ as Nystr\"{o}m landmarks $\{ \widetilde{\bm x}_1, \widetilde{\bm x}_2, \cdots, \widetilde{\bm x}_M \}$. \\
	Obtain the KRR-Nystr\"{o}m estimator $\widetilde{k}_{M, \mathcal{V}_c, \lambda}$ on the subset $\mathcal{V}_c$ by Eq.~\eqref{nysdckrr}. \\
	//{\tt end parallelism}\\
	computer the final estimator by averaging: $\bar{k}_{M,\lambda}(\bm x, \bm x') = \frac{1}{v}\sum_{c=1}^v \widetilde{k}_{M, \mathcal{V}_c, \lambda}(\bm x, \bm x')$.\\
\end{algorithm}

\subsection{The Used Hyper-kernels}
The remaining question with respect to our regression models is how to choose the hyper-kernel $\underline{k}$.
It can be noticed that numerous kernels, either PD or indefinite, can be flexibly learned in hyper-RKHS associated with a given hyper-kernel.
That means the learned kernel can be data-specific rather than manually designed. Specifically, the learning behavior is independent of the choices of hyper-kernel and the associated kernel parameters, but approximation performance on
specific data indeed relies on them.
Following \citep{kondor2007gaussian}, we adopt two hyper-kernels including the Gaussian hyper-kernel and Wishart hyper-kernel in this paper.
The Gaussian hyper-kernel is defined as
\begin{equation*}
\begin{split}
& \underline{k} \Big( (\bm x_1, \bm x_1'),  (\bm x_2, \bm x_2') \Big) \!=\! \langle \bm x_1, \bm x_1' \rangle_{{\sigma}^2} \langle \bm x_2, \bm x_2' \rangle_{{\sigma}^2}
\times \Big \langle \frac{\bm x_1 + \bm x_1'}{2}, \frac{\bm x_2 + \bm x_2'}{2} \Big \rangle_{{\sigma}^2+{\sigma}_h^2},~~ \forall \bm x_1, \bm x_1', \bm x_2, \bm x_2' \in X \,,
\end{split}
\end{equation*}
with the notation $\langle \bm x, \bm x' \rangle_{{\sigma}^2}=\frac{1}{(2\pi {\sigma}^2)^{d/2}}e^{-\| \bm x - \bm x'\|/(2{\sigma}^2)}$ as the Gaussian kernel, where $d$ is the feature dimension and ${\sigma}_h$ controls the relevance between the pairs.
We can see that this hyper-kernel not only considers the similarity between two points but also takes the similarity computed by the mean of two pairs into consideration, which is useful to enhance the representation ability of the learned kernels.
If we take the limit $\sigma_h \rightarrow \infty$, the Gaussian hyper-kernel decouples into the product of two Gaussian kernels.

Different from the Gaussian hyper-kernel that has a locally isotropic character, the Wishart hyper-kernel \citep{kondor2007gaussian} is an anisotropic one to hold for rescaling data structure, defined as 
	\begin{equation*}
	\begin{split}
	& \underline{k} \Big( (\bm x_1, \bm x_1'),  (\bm x_2, \bm x_2') \Big) \!=\! \int_{\bm \Sigma \succeq 0} \int_{X} \langle \bm x_1, \bm x \rangle_{{\bm \Sigma}}~ \langle \bm x, \bm x_1' \rangle_{{\bm \Sigma}}~ \langle \bm x_2, \bm x \rangle_{\bm \Sigma}~ \langle \bm x, \bm x_2' \rangle_{{\bm \Sigma}}~ \mathcal{IW}(\bm \Sigma;\bm D,b) \mbox{d} \bm x \mbox{d} \bm \Sigma \,,
	\end{split}
	\end{equation*}
	where
	\begin{equation*}
	\langle \bm x_1, \bm x \rangle_{{\bm \Sigma}} = \frac{1}{(2\pi)^{m/2} | \bm \Sigma |^{1/2} } \exp \Big( -\frac{1}{2}(\bm x - \bm x')^{\!\top} \bm \Sigma^{-1} (\bm x - \bm x') \Big) \,,
	\end{equation*}
	and $\mathcal{IW}(\bm \Sigma;\bm D,b)$ is the inverse Wishart distribution with the parameter matrix $\bm D \in \mathbb{R}^{m \times m}$ and an integer parameter $b$.
	The notion $\bm \Sigma \succeq 0$ means PSD matrices.
	The Wishart hyper-kernel can be regarded as the anisotropic version of the Gaussian hyper-kernel by taking $\sigma_h \rightarrow \infty$, and its second argument can be also linked to Bhattacharyya kernel \citep{kondor2003kernel} for any fixed pair.

Based on above descriptions, we formulate two regression algorithms in hyper-RKHS to output PD or indefinite kernels, which is demonstrated by stage 1 in Figure~\ref{framework}.
Then the following classification task to learn the hypothesis $f$ in stage 2 can be achieved by kernel machines, e.g., SVM used in this paper.
Note that, if the learned kernel is indefinite, the SVM solver is still valid, but outputs a stationary point instead of the optimal minimum.
In fact, we can also choose some advanced algorithms in RKKS, e.g., \citep{Ga2016Learning,oglic18a}, as alternative ways for learning in RKKS.

\section{Generalization Properties of Learning in Hyper-RKHS}
\label{sec:learningrate}

In this section, we study the convergence analysis of learning problems in hyper-RKHS with squared loss and $\varepsilon$-insensitive loss.
Although approximation analysis of classical regression algorithms including least-squares regularized regression \citep{Wu2006Learning,caponnetto2007optimal,dieuleveut2017harder}, support vector regression \citep{Xiang2012Approximation}, quantile regression \citep{Shi2014Quantile} in RKHS are provided, the generalization properties (in an approximation theory view) of regression problems in hyper-RKHS have not yet been fully investigated.

\subsection{Problem Settings and Notations}
In the context of statistical learning theory, to investigate a general regularized regression problem in hyper-RKHS, the learned regression function, also the kernel function $k$ is defined on a compact metric space $ {X} \times {X} $ denoted by ${\underline{X}}$. The hyper-kernel $\underline{k}: \underline{{X}} \times \underline{{X}} \rightarrow \mathbb{R}$ is a continuous, symmetric, positive definite function. Then
the associated hyper-RKHS $\underline{\mathcal{H}}$ in Definition~\ref{definitehyperrkhs} is the completion of the linear span of the set of function $\{ \underline{k}(\underline{\bm x},\cdot) : \underline{\bm x} \in \underline{X} \}$ with the inner product  $\langle \cdot, \cdot \rangle_\mathcal{\underline{H}}$.

Let $\rho$ be a non-degenerate Borel probability measure on $Z = {\underline{X}} \times Y = {X} \times {X} \times Y$ which can be factorized as
\begin{equation*}
  \rho(\underline{{\bm x}},y) = \rho(\bm x, \bm x', y) = \rho_{{X}}(\bm x) \rho_{{X}}(\bm x') \rho(y|(\bm x, \bm x'))\,,
\end{equation*}
where $\rho_{{X}}$ is a probability measure on ${X}$ and $\rho(y|(\bm x, \bm x'))$ is the conditional distribution on $Y$ given $(\bm x, \bm x') \in \underline{{X}}$.
The \emph{target function} (also a kernel function) of $\rho$ is defined by
\begin{equation*}
 k_{\rho}(\bm x, \bm x') = \int_Y y \mbox{d} \rho\big(y|(\bm x, \bm x')\big)\,, \quad \bm x, \bm x' \in X\,.
\end{equation*}
The target of regression problem in hyper-RKHS is to find a good approximation of $k_{\rho}$ from the pairwise sample set $Z=\{\bm z_{ij}\}_{i,j=1}^{m} = \{ (\bm x_i, \bm x_j, y_{ij}) \}_{i,j=1}^{m}$, where $\{ \bm x_i\}_{i=1}^m$ are sampled independently according to $\rho_{{X}}$ and $y_{ij}$ is drawn from the conditional distribution $\rho(y|(\bm x, \bm x'))$. Note that these $m^2$ pairwise samples $\{ (\bm x_i, \bm x_j, y_{ij}) \}_{i,j=1}^{m}$ are not mutual pairwise independent \citep{Luby2005Pairwise}.
Actually, for $i \neq j$, $\bm z_{ij}$ is drawn according to $\rho$, while $\bm z_{ii}$ is distributed according to $\rho'({{\bm x}},y) = \rho_{{X}}(\bm x) \rho(y|(\bm x, \bm x))$.

The target function $k_{\rho}$ is estimated by minimizing the expected risk
\begin{equation*}
\begin{split}
 &  \mathcal{E}(k) := \mathcal{E}_{\rho}(k) = \int_{\underline{{X}} \times Y} \mathcal{T}\big(y, k(\underline{\bm x})\big) \mbox{d} \rho
 =\!\int_{{X}} \! \int_{{X}} \! \int_{Y}\! \mathcal{T}\big(k(\bm x, \bm x'), y\big) \mbox{d} \rho(y|(\bm x, \bm x'))  \mbox{d} \rho_{{X}}(\bm x)  \mbox{d} \rho_{{X}}(\bm x')\,.
  \end{split}
\end{equation*}
We additionally suppose that there exits a constant $M^* \geq 1$, such that
\begin{equation*}\label{assumption1}
|k_{\rho}(\bm x, \bm x')|\leq M^* \mbox{ for all $\bm x, \bm x' \in{X}$}\,.
\end{equation*}

For the squared loss, we have $\mathcal{T}(y, k({\bm x},\bm x')) = \big( y - k({\bm x},\bm x') \big)^2 $, and thus the empirical risk functional is defined as
\begin{equation*}
  \mathcal{E}_{\bm z}(k) = \frac{1}{m^2} \sum_{i,j=1}^{m} \mathcal{T}(y_{ij}, k({\bm x_i},\bm x_j))  = \frac{1}{m^2} \sum_{i,j=1}^{m} \Big(k(\bm x_i, \bm x_j) - y_{ij} \Big)^2\,.
\end{equation*}
Hence, given the sample set $Z$, KRR in hyper-RKHS aims at finding a kernel function $k$ : ${X} \times {X} \rightarrow \mathbb{R}$ such that $k_{\bm{z}}({\bm x}, \bm x')$ is a good estimate of $y$ for a new pair input $({\bm x}, \bm x')$.
To be specific, the learning algorithm generated by regularized least squares in hyper-RKHS takes the form
\begin{equation}\label{fzrsls}
  k_{\bm{z},\lambda} := \argmin_{k \in \underline{\mathcal{H}}} \bigg\{ \frac{1}{m^2} \sum_{i,j=1}^{m} \!\! \Big(k(\bm x_i, \bm x_j) - y_{ij} \Big)^2 + \lambda\langle k,k\rangle_{\mathcal{\underline{H}}} \bigg\}.
\end{equation}

For SVR in hyper-RKHS, it is a little sophisticated due to the insensitive parameter $\varepsilon$.
Here we consider SVR with $\mathcal{T}(y, k({\bm x},\bm x')) = | y - k({\bm x},\bm x')| $, and then introduce $\varepsilon$-insensitive loss in SVR.
For any $\bm x, \bm x' \in X$, the target kernel function $k_{\rho}$ is defined by its value $k_{\rho}({\bm x}, \bm x')$ to be a median function of $\rho(\cdot|(\bm x, \bm x'))$, that is
\begin{equation*}
\left\{
\begin{array}{rcl}
\begin{split}
  &\rho \Big(\big\{ y \in Y: y \leq k_{\rho}({\bm x},\bm x') \big\} \big| (\bm x, \bm x') \Big) \geq \frac{1}{2}\,,\\
    & \rho \Big(\big\{ y \in Y: y \geq k_{\rho}({\bm x},\bm x') \big\} \big| (\bm x, \bm x') \Big) \geq \frac{1}{2}\,.
\end{split}
\end{array} \right.
\end{equation*}

In order to obtain a sparse solution, we introduce the $\varepsilon$-insensitive loss function in SVR
\begin{equation*}
\mathcal{T}^{\varepsilon}\big(y, k(\bm x, \bm x')\big) \!=\! | y - k(\bm x, \bm x')|_{\varepsilon} \!=\! \left\{
\begin{array}{rcl}
\begin{split}
& 0,  ~\text{if}~ | y - k(\bm x, \bm x')| < \varepsilon; \\
& | y - k(\bm x, \bm x')| - \varepsilon, ~\text{otherwise}\,, \\
\end{split}
\end{array} \right.
\end{equation*}
where the insensitivity parameter $\varepsilon$ aims at balancing the approximation and sparsity
of the algorithm and thus should change with the sample size $m$ satisfying $ \lim_{m \rightarrow \infty} \varepsilon(m) = 0$.
Given the sample set $Z$, SVR in hyper-RKHS with the $\varepsilon$-insensitive loss takes the form
\begin{equation}\label{fzrs}
  k^{(\varepsilon)}_{\bm{z},\lambda}\! :=\! \argmin_{k \in \underline{\mathcal{H}}} \bigg\{ \! \frac{1}{m^2} \! \sum_{i,j=1}^{m} \!\! \mathcal{T}^{\varepsilon}\big(y_{ij}, k(\bm x_i,\bm x_j)\big) \!+\! \lambda\langle k,k\rangle_{\mathcal{\underline{H}}} \bigg\}.
\end{equation}


\subsection{Definitions and Assumptions}
To illustrate the convergence analysis, we need the following definitions and assumptions.
Note that all of the presented assumptions in hyper-RKHS here are defined on pairs but can be analogous to that in RKHS, and are hence standard and fair in approximation analysis.

We first state the definition of \emph{projection operator} introduced in \citep{Chen2004Support}.
\begin{definition}\label{proj}
(projection operator)
For $B > 0$, the projection operator $\pi = \pi_{B}$ is defined on the space of measurable functions $k: X \times X \rightarrow \mathbb{R}$ as
\begin{equation*}\label{BBPdef}
\pi_B(k)(\bm x, \bm x')= \left\{
\begin{array}{rcl}
\begin{split}
& B,  ~~\text{if}~~ k(\bm x, \bm x') > B ; \\
&-B, ~~\text{if}~~k(\bm x, \bm x') < -B ; \\
& k(\bm x, \bm x'), ~~\text{if}~~-B \leq k(\bm x, \bm x') \leq B\,,
\end{split}
\end{array} \right.
\end{equation*}
and then the projection of $k$ is denoted as $\pi_B(k)(\bm x, \bm x') = \pi_B(k(\bm x, \bm x'))$.
\end{definition}
Since $k_{\rho}$ takes the value in $[-M^*, M^*]$ almost surely, the projection operator is beneficial to estimate $k_{\rho}$ by $\pi_{M^*}(k^{(\varepsilon)}_{\bm{z},\lambda})$ instead of $k^{(\varepsilon)}_{\bm{z},\lambda}$ for sharp estimation.
Therefore, for SVR in hyper-RKHS, our approximation analysis attempts to bound the error $\| \pi_{M^*}(k^{(\varepsilon)}_{\bm{z},\lambda})  - k_{\rho} \|_{L_{\rho_{{X}}}^{p^*}} $ in the space ${L_{\rho_X}^{2}}(\underline{X})$ with some $p^*>0$, where $L^p_{\rho_{{X}}}$ is a weighted $L^p$-space with the norm
\begin{equation*}
  \|k\|_{L^p_{\rho_{{{X}}}}} = \Big( \int_{{X}} \int_{{X}} |k(\bm x, \bm x')|^p \mbox{d} \rho_{X}(\bm x) \mbox{d} \rho_{X}(\bm x') \Big)^{1/p}\,.
\end{equation*}

To estimate the approximation error, we need the following assumptions with respect to the unbounded outputs, noise condition on $\rho$, and covering numbers for the hypothesis space.
Here we consider a general setting with respect to the unbounded outputs \citep{Wang2011Optimal}.
\begin{definition}\label{defoutput}
(moment hypothesis)
There exist constants $M \geq 1$ and $c>0$ such that
\begin{equation}\label{Momenthypothesis}
\int_Y |y|^\iota \emph{d} \rho(y|(\bm x, \bm x')) \leq c \iota ! M^\iota, ~ \forall \iota
\in \mathbb{N},~ (\bm x, \bm x') \in \underline{X}\,.
\end{equation}
\end{definition}
{\bf Remark:} Compared to the standard uniform boundedness assumption with $|y| \leq M$ almost surely, this assumption is general since it covers Gaussian noise, sub-Gaussian noise, etc.
If the condition distribution $\rho(\cdot|(\bm x, \bm x'))$ is a Gaussian distribution with variance $\sigma_X^2$ bounded by $B_0$, then Eq.~\eqref{Momenthypothesis} is satisfied with $M:=\max\{\sqrt{2}B_0, M^*\}$ and $c=4$.

The noise condition on $\rho$  \citep{Christmann2007How} via pairs can be defined in a similar fashion with that in RKHS.
\begin{definition}\label{def2}
(noise condition)
Let $p \in (0,\infty]$ and $q\in [1, \infty)$. A distribution $\rho$
on ${X} \times {X} \times R$ is said to have a median of $p-$average type
$q$ if for any $(\bm x, \bm x' )\in \underline{X}$, there exist a
median $t^*$ and constants $0<a_{(\bm x, \bm x')}\leq 1$, $b_{(\bm x, \bm x')}>0$
such that for each $u\in [0,a_{(\bm x, \bm x')}]$,
\begin{equation}\label{noisecondition}
\left\{
\begin{array}{rcl}
\begin{split}
&\rho((t^*-u,t^*)|{(\bm x, \bm x')})\geq b_{(\bm x, \bm x')} u^{q-1}\\
&\rho((t^*,t^*+u)|{(\bm x, \bm x')})\geq b_{(\bm x, \bm x')} u^{q-1}\,,
\end{split}
\end{array}\right.
\end{equation}
and that the function on ${X} \times {X}$ taking values $\big(b_{(\bm x, \bm x')} a_{(\bm x, \bm x')}^{q-1}\big)^{-1}$ at $(\bm x, \bm x' )\in X \times X$
lies in $L^p_{\rho_{{X}}}$.
\end{definition}
The noise condition in Eq.~\eqref{noisecondition} ensures that $k_{\rho}(\bm x, \bm x') = t^*$ is uniquely defined at every $(\bm x, \bm x' )\in \underline{X}$.

Apart from the above conditions, our main results about learning rates also involve the approximation ability of $\underline{\mathcal{H}}$ with respect to its capacity and $k_{\rho}$.
The approximation ability can be characterised by the regularization error.
\begin{definition}\label{defapperr}
The regularization error is defined as
\begin{equation}\label{Dlamdadef}
D(\lambda)=\inf_{k \in \mathcal{\underline{H}}} \Big\{ \mathcal{E}(k) - \mathcal{E}(k_{\rho}) + \lambda \| k\|_{\underline{\mathcal{H}}}^2 \Big\} \,.
\end{equation}
The target kernel function $k_{\rho}$ can be approximated by $\mathcal{\underline{H}}$ with exponent $0 < r \leq 1$ if there exists a constant $C_0$ such that
 \begin{equation}\label{Dlambda}
   D(\lambda) \leq C_0\lambda^{r},~~\forall \lambda>0\,.
 \end{equation}
\end{definition}
\noindent{\bf Remark:} This is a natural assumption in approximation theory, e.g., \citep{Wu2006Learning,Wang2011Optimal,Steinwart2008SVM}.
Note that $r=1$ is the best choice as we expect, which is equivalent to $k_{\rho} \in \underline{\mathcal{H}}$ when $\underline{\mathcal{H}}$ is dense.
In fact, the assumption in Eq.~\eqref{Dlambda} can be also characterized by the \emph{source condition} via integral operator, refer to \citep{caponnetto2007optimal} for details.

Further, to quantitatively understand that how the complexity of $\underline{\mathcal{H}}$ affects the learning ability of algorithm in Eq.~\eqref{fzrs}, we need the capacity (roughly speaking the ``size'') of $\underline{\mathcal{H}}$ measured by covering numbers \citep{cucker2007learning}.
\begin{definition}\label{defcovering}
For a subset $S$ of $C(\underline{{X}})$ and $\epsilon> 0$, the \emph{covering number} $\mathscr{N}(S, \epsilon)$ is the minimal
integer $l \in \mathbb{N}$ such that there exist $l$ disks with radius $\epsilon$ covering $S$.
\end{definition}
In this paper, the covering numbers of balls are defined by
\begin{equation*}\label{BRradius}
 \mathcal{B}_R = \{ k \in \mathcal{\underline{H}}: \| k\|_{\underline{\mathcal{H}}} \leq R \}\,,
\end{equation*}
where we assume that for some $s>0$ and $C_s>0$ such that
 \begin{equation}\label{assumpN}
 \log \mathscr{N}(\mathcal{B}_1,\epsilon) \leq C_s \Big(\frac{1}{\epsilon}\Big)^s, \quad \forall \epsilon>0\,.
 \end{equation}
{\bf Remark:} This is a standard assumption to measure the capacity of $\underline{\mathcal{H}}$ that follows with RKHS \citep{cucker2007learning,Wang2011Optimal,shi2019sparse}. When $\underline{X}$ is a bounded domain and $\underline{k} \in C^{\tau}(\underline{X} \times \underline{X})$, Eq.~\eqref{assumpN} holds true with $s={2m^2}/{\tau}$. In particular, if $\underline{k} \in C^{\infty}(\underline{X} \times \underline{X})$, condition~\eqref{assumpN} is valid for an arbitrarily small $s>0$.
		In fact, the capacity of a (hyper)-RKHS can be also measured by eigenvalue decay of the reproducing (hyper)-kernel matrix $\underline{\bm K}$ or effective dimension in integral operator theory \citep{caponnetto2007optimal}. As demonstrated by \citep{bach2013sharp,belkin2018approximation}, a small (hyper)-RKHS often indicates a fast eigenvalue decay so as to obtain a promising prediction performance. In other words, functions in the (hyper)-RKHS are potentially smoother than what is necessary, which means an arbitrary small $s$ in Eq.~\eqref{assumpN}.

\subsection{Main Results}
Formally, our main results about SVR in hyper-RKHS are stated as follows.
For $p\in (0,\infty]$ and $q\in (1,\infty)$, we denote
\begin{equation}\label{theta}
\theta=\min \left\{\frac{2}{q},\frac{p}{p+1}\right\}\in (0,1]\,.
\end{equation}
\begin{theorem}\label{theorem1s}
Suppose that $| k_{\rho}(\bm x, \bm x')| \leq M^*$ with $M^* \geq 1$, $\rho$ has a median of $p$-average type $q$ with some $p\in (0,\infty]$ and $q\in (1,\infty)$ and satisfies assumptions Eq.~\eqref{Dlambda} with $0 < r \leq 1$ and Eq.~\eqref{Momenthypothesis}.
Assume that for some $s>0$,  take $\lambda=m^{-\alpha}$ with $ 0 < \alpha \leq 1$ and $\alpha < \frac{1+s}{s(2+s-\theta)}$, and set $\varepsilon = m^{-\gamma}$ with $\alpha r \leq \gamma \leq \infty$. Then with $p^*=\frac{pq}{p+1}$, for any $0 < \epsilon < \Theta/q$ and $0<\delta<1$, with confidence $1-\delta$,
we have
\begin{equation*}\label{result1}
\begin{split}
\big\| \pi_{M^*}(k^{(\varepsilon)}_{\bm{z},\lambda})  - k_{\rho} \big\|_{L_{\rho_{{X}}}^{p^*}}  \leq  \widetilde{C}^{\epsilon}_{\underline{X}, \rho, \alpha, \gamma} \left( \log\frac{4}{\delta} \right)^{1/q} m^{\epsilon - \frac{\Theta}{q}} \,,
\end{split}
\end{equation*}
where $\widetilde{C}^{\epsilon}_{\underline{X}, \rho, \alpha, \gamma}$ is a constant independent of $m$ or $\delta$ and the power index $\Theta$ is
\begin{equation}\label{rate1}
\begin{split}
  \Theta
 & = \min \left\{ {\alpha r},   \frac{1+\alpha r -\alpha}{2-\theta}
 , \frac{1}{2+s-\theta} - \frac{\alpha s}{1+s}  \right\}\,.
  \end{split}
\end{equation}
\end{theorem}
The power index $\Theta$ can be viewed as a function of variables $r,s,p,q,\alpha$.
The restriction $\alpha < \frac{1+s}{s(2+s-\theta)}$ ensures that $\Theta$ is positive, which verifies the valid learning rate in Theorem~\ref{theorem1s}. \\
{\bf Remark:} Note that $s$ can be arbitrarily small when the hyper-kernel $\underline{k}$ is smooth enough.
In this case, the power index $\Theta$ in Eq.~\eqref{rate1} can be arbitrarily close to $\min ({\alpha r},  \frac{1}{2-\theta}+\frac{r-1}{2-\theta}\alpha)$.
Regarding to SVR in RKHS, \cite{Xiang2012Approximation} demonstrates that the power index $\Theta$ in Eq.~\eqref{rate1} can be arbitrarily close to $\min ({\alpha r},  \frac{1}{2-\theta})$ when the reproducing kernel is smooth enough. In this case, the derived learning rate in hyper-RKHS is not faster than that in RKHS, which is mainly effected by the approximation ability since the spanning space by hyper-RKHS is larger than RKHS.
Nevertheless, if we further consider $k_{\rho} \in \underline{\mathcal{H}}$, that means the approximation error in Eq.~\eqref{Dlambda} can be upper bounded with $r=1$, the derived learning rate in hyper-RKHS is the same as that in RKHS, approaching to $\min ({\alpha},  \frac{1}{2-\theta})$.

Regarding to KRR in hyper-RKHS, the excess error for squared loss is exactly the distance in the space ${L_{\rho_X}^{2}}(\underline{X})$, i.e., $\mathcal{E}(k) - \mathcal{E}(k_{\rho}) =  \big\| k - k_{\rho} \big\|^2_{L_{\rho_{{X}}}^{2}}$, which yields a direct variance-expectation bound.
Our results about least-squares in hyper-RKHS are presented as follows.
\begin{theorem}\label{theorem1ls}
Suppose that $| k_{\rho}(\bm x, \bm x')| \leq M^*$ with $M^* \geq 1$, $\rho$ satisfies the condition in Eq.~\eqref{Dlambda} with $0 < r \leq 1$ and the moment hypothesis in Eq.~\eqref{Momenthypothesis} with $c>0$.
Assume that for some $s>0$,  take $\lambda=m^{-\alpha}$ with $ 0 < \alpha \leq 1$ and $\alpha < \frac{1+s}{s(2+s)}$. Then for any $0 < \epsilon < \Theta$ and $0<\delta<1$, with confidence $1-\delta$,
we have
\begin{equation*}\label{result1ls}
\begin{split}
\big\| \pi_{M^*}(k_{\bm{z},\lambda})  - k_{\rho} \big\|_{L_{\rho_X}^{2}}  \leq  \widetilde{C}_{\underline{X}, \rho, \alpha, \gamma}  \log\frac{4}{\delta} m^{\epsilon - {\Theta}} \,,
\end{split}
\end{equation*}
where $\widetilde{C}_{\underline{X}, \rho, \alpha, \gamma}$ is a constant independent of $m$ or $\delta$ and the power index $\Theta$ is
\begin{equation}\label{rate1ls}
\begin{split}
  \Theta
 & = \min \left\{ {\alpha r},  \frac{1}{2+s} - \frac{\alpha s}{1+s}  \right\}\,.
  \end{split}
\end{equation}
\end{theorem}
{\bf Remark:} In the special case that $k_{\rho} \in \underline{\mathcal{H}}$ (i.e., $r=1$) and $\underline{k} \in C^{\infty}(\underline{X} \times \underline{X})$, condition~\eqref{assumpN} is satisfied for an arbitrarily small $s>0$. Accordingly, the excess error $\big\| \pi_{M^*}(k_{\bm{z},\lambda})  - k_{\rho} \big\|^2_{L_{\rho_X}^{2}}$ can converge to zero at the (arbitrary close to) optimal rate $\mathcal{O}(1/m)$ if we take $\alpha \geq 1/2$, which matches to results on least squares in RKHS under the same assumptions, e.g., Theorem~1 in \citep{Wang2011Optimal}, and Corollary 1 in~\citep{guo2013concentration}.

\section{Framework of Proofs}
\label{sec:prooframes}
In this section, we establish the framework of proofs for Theorem~\ref{theorem1s}.
We use the error decomposition technique to analyze the convergence behavior of SVR in hyper-RKHS.
The key challenges in our theoretical analyses include analyzing the bias of the estimator, the effect of noise on the unbounded outputs, the non-trivial independence of pairwise samples, and the characterisation of hyper-RKHS.
The last two points are the main elements on novelty in the proof.
Since the proofs about the learning rate for least-squares regression in hyper-RKHS can be regarded as a simplified version of SVR in hyper-RKHS, we concentrate our proof on $\varepsilon$-insensitive loss in hyper-RKHS and omit the detailed proofs for the squared loss.

Before proving Theorem~\ref{theorem1s}\,, we need the proposition introduced in \citep{Steinwart2011Estimating,Shi2014Quantile}.

\begin{proposition}\label{lemmafe}
Suppose that $| k_{\rho}(\bm x, \bm x')| \leq M^*$ with $M^* \geq 1$, $\rho$ has a median of $p$-average type $q$ with some $p\in (0,\infty]$ and $q\in (1,\infty)$, for any
$k:{X} \times {X}\to [-B,B]$ with $B>0$, there holds
\begin{equation}\label{boundfe}
\big\| k  - k_{\rho} \big\|_{L_{\rho_X}^{p^*}}
\leq C_q\max\{M^*,B\}^{1-1/q} \Big\{\mathcal{E}(k) -\mathcal{E}(k_{\rho})\Big\}^{\frac{1}{q}},
\end{equation}
where $p^* = \frac{pq}{p+1}$ and $C_{q}=2^{1-1/q}q^{1/q}
\|\{(b_{(\bm x, \bm x')}a_{(\bm x, \bm x')}^{q-1})^{-1}\}_{(\bm x, \bm x')\in {{X} \times {X}}}\|^{1/q}_{L^p_{\rho_X}}$.
\end{proposition}
This proposition demonstrates that the excess error $\mathcal{E}(k) -\mathcal{E}(k_{\rho})$ can be analysed by $k_{\rho}$ and its approximation $k$ in $L_{\rho_{\underline{X}}}^{p^*}$.
Instead, the excess error for squared loss is exactly the distance in the space ${L_{\rho_X}^{2}}(\underline{X})$, i.e., $\mathcal{E}(k) - \mathcal{E}(k_{\rho}) =  \big\| k - k_{\rho} \big\|^2_{L_{\rho_{{X}}}^{2}}$.

\subsection{Error Decomposition}
\label{sec:ed}
In order to estimate error $\| \pi_{M^*}(k^{(\varepsilon)}_{\bm{z},\lambda})  - k_{\rho} \|$ in the $L_{\rho_X}^{p^*}$ space, i.e., to bound $\| \pi_{B}(k^{(\varepsilon)}_{\bm{z},\lambda})  - k_{\rho}  \|$ for any $B \geq M^*$. Accordingly, by Proposition~\ref{lemmafe}\,, we need to estimate the excess error $\mathcal{E}\big(\pi_B(k^{(\varepsilon)}_{\bm{z},\lambda})\big) - \mathcal{E}(k_{\rho})$ which can be conducted by an error decomposition technique \citep{cucker2007learning}.
Note that the insensitivity parameter $\varepsilon$ changes with $m$, we consider the insensitivity relation with additional $\varepsilon$ on the error decomposition \citep{Xiang2011Learning}, that is
\begin{equation}\label{insentivec}
 \mathcal{T}(y, k(\bm x, \bm x')) - \varepsilon \leq \mathcal{T}^{\varepsilon}(y, k(\bm x, \bm x')) \leq \mathcal{T}(y, k(\bm x, \bm x')), \quad \forall \bm x, \bm x' \in X\,.
\end{equation}

Formally, the error decomposition is given by the following proposition, with proof deferred to Appendix~\ref{sec:errdec}.
\begin{proposition}\label{errdec}
Let
\begin{equation*}
k_\lambda = \argmin_{k \in \underline{\mathcal{H}}} \Big\{ \mathcal{E}(k) - \mathcal{E}(k_{\rho}) + \lambda \| k\|_{\underline{\mathcal{H}}}^2 \Big\} \,.
\end{equation*}

Then the excess error $\mathcal{E}\big(\pi_B(k^{(\varepsilon)}_{\bm{z},\lambda})\big) - \mathcal{E}(k_{\rho})$ can be bounded by
 \begin{equation*}
 \begin{split}
    \mathcal{E}\big(\pi_B(k^{(\varepsilon)}_{\bm{z},\lambda})\big) - \mathcal{E}(k_{\rho}) & \leq  \mathcal{E}\big(\pi_B(k^{(\varepsilon)}_{\bm{z},\lambda})\big) - \mathcal{E}(k_{\rho}) + \lambda \| k^{(\varepsilon)}_{\bm{z},\lambda} \|^2_{\underline{\mathcal{H}}} \\
     & \leq D(\lambda) + S(\bm z, \lambda) + \frac{1}{m^2} \sum_{i,j=1}^{m} \Big| \pi_B(y_{ij}) - y_{ij} \Big|  + \varepsilon\,,
    \end{split}
 \end{equation*}
 where $D(\lambda)$ is the regularization error defined by Eq.~\eqref{Dlamdadef}.
The sample error $S(\bm z, \lambda)$ is denoted as
\begin{equation*}
\begin{split}
&S(\bm z, \lambda)  = \mathcal{E}\big(\pi_B(k^{(\varepsilon)}_{\bm{z},\lambda})\big) - \mathcal{E}_{\bm{z}}\big(\pi_B(k^{(\varepsilon)}_{\bm{z},\lambda})\big) + \mathcal{E}_{\bm{z}}\big(k_{\lambda}\big) - \mathcal{E}\big(k_{\lambda}\big) = S_1(\bm z, \lambda) + S_2(\bm z, \lambda) \\
\end{split}
\,,
\end{equation*}
with
\begin{equation*}
\begin{split}
 S_1(\bm z, \lambda) &= \Big\{ \mathcal{E}\big(\pi_B(k^{(\varepsilon)}_{\bm{z},\lambda})\big) -  \mathcal{E}(k_{\rho}) \Big\} - \Big\{ \mathcal{E}_{\bm{z}}\big(\pi_B(k^{(\varepsilon)}_{\bm{z},\lambda})\big) - \mathcal{E}_{\bm{z}}(k_{\rho}) \Big\}\,,\\
  S_2(\bm z, \lambda) &= \Big\{ \mathcal{E}_{\bm{z}}\big(k_{\lambda}\big) - \mathcal{E}_{\bm{z}}(k_{\rho}) \Big\} - \Big\{ \mathcal{E}(k_{\lambda}) - \mathcal{E}(k_{\rho}) \Big\}\,.
  \end{split}
\end{equation*}
 \end{proposition}
By Proposition~\ref{errdec}\,, the excess error can be bounded by the sample error $S(\bm z, \lambda)$, the regularization error $D(\lambda)$, and the output error.
The regularization error is bounded by Eq.~\eqref{Dlambda}.
Besides, by virtue of Eq.~\eqref{fnormdiff} and Eq.~\eqref{Dlamdadef}, $k_{\lambda}$ under supremum norm can be also upper bounded by
	\begin{equation}\label{flambound}
	\| k_{\lambda}\|_{\infty} \leq \mathcal{G} \| k_{\lambda}\|_{\underline{\mathcal{H}}} \leq \mathcal{G} \sqrt{\frac{D(\lambda)}{\lambda}} \leq \mathcal{G} \sqrt{C_0}\lambda^{\frac{r-1}{2}}\,.
	\end{equation}	
In the next, our error analysis mainly focuses on how to estimate the sample error and the output error.
We expect that these approximation errors will approximate to zero at a certain rate as the sample size tends to infinity.

\subsection{Estimate Sample Error and Output Error}
\label{sec:se}
This section is devoted to estimating the sample error $S(\bm z, \lambda)$ and the output error.
 Our error analysis mainly focuses on how to estimate $S_1(\bm z, \lambda)$ and $S_2(\bm z, \lambda)$.
The asymptotical behaviors of $S_1$ and $S_2$ are usually illustrated by
the convergence of the empirical mean $\frac{1}{m^2}\sum_{i,j=1}^m\xi_{ij}$
to its expectation $\mathbb{E}\xi$, where
$\left\{\xi_{ij}\right\}_{i,j=1}^m$ are ``independent" random variables on
$(Z,\rho)$ defined as
\begin{equation}\label{randomvariables}
\xi(\bm x, \bm x', y) :=\mathcal{T}\big(y,k_{\lambda}(\bm x, \bm x')\big) - \mathcal{T}\big(y,k_{\rho}(\bm x, \bm x')\big)\,,~{\text{and}}~ \xi_{ij} := \xi(\bm x_i, \bm x_j, y_{ij})\,.
\end{equation}
Note that the Lipschitz property of the $\varepsilon$-insensitive loss in SVR guarantees the boundedness of $\xi$ when $k$ is bounded.
So $\xi$ defined by Eq.~\eqref{randomvariables} is a bounded random variable even if
$y$ is unbounded.
When $k$ is fixed, which is exactly the case as we estimate $S_2$,
the convergence is guaranteed by the following lemma.

For $R\geq 1$, denote
\begin{equation}\label{set}
\mathscr{W}(R)=\left\{{\bm z}\in Z^{m \times m}: \|k^{(\varepsilon)}_{\bm{z},\lambda}\|_{\underline{\mathcal{H}}}\leq R\right\}.
\end{equation}

\begin{lemma}\label{lemma4}
If $\xi$ is a symmetric real-valued function on $X \times X \times Y$ with mean $\mathbb{E}(\xi)$. Assume that $\mathbb{E}(\xi) \geq 0$, $|\xi (\bm x, \bm x', y)- \mathbb{E}\xi|\leq T$ almost surely and $\mathbb{E}\xi^2 \leq c_1 (\mathbb{E}\xi)^{\theta}$ for some $0 \leq \theta \leq 1$ and $c_1 \geq 0$, $T \geq 0$.
Then for every $\epsilon>0$ there holds
\begin{equation}\label{leamm4eq}
\mathop{\rm Prob} \left\{
\frac{\frac{1}{m(m-1)}\sum_{i=1}^m \sum_{j=1, j \neq i}^{m}
\xi(\bm x_i, \bm x_j, y_{ij}) - \mathbb{E}\xi}{{\sqrt{(\mathbb{E}\xi)^{\theta}+\epsilon^{\theta}}}} \geq
\epsilon^{1-\frac{\theta}{2}}\right\}\leq
\exp\left\{-\frac{(m-1)\epsilon^{2-\theta}}{4c_1+\frac{4}{3}T \epsilon^{1-\theta}}\right\}.
\end{equation}
\end{lemma}
\begin{proof}
Define
\begin{equation*}
U_i = \frac{1}{2 \lfloor m/2 \rfloor} \left\{ \xi_{i_1 i_2} + \xi_{i_2 i_1} + \xi_{i_3 i_4} + \xi_{i_4 i_3} + \cdots + \xi_{i_{2\lfloor m/2 \rfloor-1}} \xi_{i_{2\lfloor m/2 \rfloor}} + \xi_{i_{2\lfloor m/2 \rfloor}} \xi_{i_{2\lfloor m/2 \rfloor-1}}   \right\}\,,
\end{equation*}
where $\lfloor m/2 \rfloor$ denotes the greatest integer not exceeding $m/2$, and $(i_1, i_2, \dots, i_m)$ is a permutation of $(1,2,\dots,m)$. Then
\begin{equation*}
  \frac{1}{m(m-1)} \sum_{i,j=1,i \neq j}^{m} \xi(\bm x_i, \bm x_j, y_{ij})  = \frac{1}{m!} \sum_{m\cdot m} U_i,~\text{and}~ \mathbb{E}(U_i) = \mathbb{E} \xi\,,
\end{equation*}
where the notation $\sum_{m\cdot m}$ is the summation taken over all permutations of the integers $1,2,\dots,m$.
Note that for distinct integers $k,k',l,l'$ not exceeding $m$, random variables $\xi_{i_k i_l}+\xi_{i_l i_k}$ and $\xi_{i_{k'} i_{l'}} \xi_{i_{l'} i_{k'}}$ are independent.
Then each $U_i$ is a summation of $\lfloor m/2 \rfloor$ independent random variables.
Therefore, we have
\begin{equation}
\begin{split}
& \mathop{\rm Prob} \left\{
\frac{\frac{1}{m(m-1)}\sum_{i=1}^m \sum_{j=1, j \neq i}^{m}
\xi(\bm x_i, \bm x_j, y_{ij}) - \mathbb{E}\xi}{{\sqrt{(\mathbb{E}\xi)^{\theta}+\epsilon^{\theta}}}}  \geq
\epsilon^{1-\frac{\theta}{2}}\right\} = \mathop{\rm Prob} \left\{
\frac{\frac{1}{m!}\sum_{m \cdot m} \Big[ U_i - \mathbb{E}\xi \Big] }{\sqrt{(\mathbb{E}\xi)^{\theta}+\epsilon^{\theta}}} \geq
\epsilon^{1-\frac{\theta}{2}}\right\} \\
& \leq  \frac{1}{m!}  \mathop{\rm Prob} \Big\{
 \frac{U_i - \mathbb{E}U_i}{{\sqrt{(\mathbb{E}\xi)^{\theta}+\epsilon^{\theta}}}} \geq
\epsilon^{1-\frac{\theta}{2}}\Big\} \leq \exp \left\{ - \frac{\lfloor m/2 \rfloor \epsilon^{2-\theta} }{2\big( c_1 + \frac{1}{3} T \epsilon^{1-\theta} \big)} \right\}\,.
\end{split}
\end{equation}
Here we derive the last inequality by applying the Bernstein inequality.
We thus complete the proof by noting $\lfloor m/2 \rfloor \geq \frac{m-1}{2}$.
\end{proof}
When the random variables are given by Eq.~\eqref{randomvariables}, for a general distribution $\rho$, the variance-expectation condition $\mathbb{E}\xi^2 \leq c_1 (\mathbb{E}\xi)^{\theta}$ is satisfied with $\theta=0$ and $c_1=1$.
Specifically, if $\rho$ satisfies the noise condition (i.e., Definition~\ref{def2}), the variance-expectation bound can be improved by the following lemma.
\begin{lemma}\label{lemmas3}
Under the same assumption of Proposition~\ref{lemmafe}\,, for any
$k:X \times X\to [-B,B]$ with $B>0$, there holds
\begin{equation}\label{bound3}
\mathbb{E}\bigg\{ \mathcal{T}\big(y,k(\bm x, \bm x')\big) -\mathcal{T}\big(y,k_{\rho}(\bm x, \bm x')\big) \bigg\}^2
\leq C_{\theta}\max \{ B,M^*\}^{2-\theta} \Big(\mathcal{E}(k) -\mathcal{E}(k_{\rho})\Big)^{\theta},
\end{equation} where $\theta$
is given by Eq.~\eqref{theta} and $C_{\theta}=2^{2-\theta}q^{\theta}
\|\{(b_{(\bm x, \bm x')}a_{(\bm x, \bm x')}^{q-1})^{-1}\}_{(\bm x, \bm x')\in {\underline{X} }}\|^{\theta}_{L^p_{\rho_X}}$ with the constant $a_{(\bm x, \bm x')} \in (0,1]$ and $b_{(\bm x, \bm x')} > 0$.
\end{lemma}
This lemma is a direct corollary of Proposition~\ref{lemmafe}\,.
The positive $\theta$ here will lead to sharper estimates and play an essential role in the convergence analysis.

Now we can bound $S_2(\bm z, \lambda)$ by the following proposition, refer to the proof in Appendix~\ref{sec:s2}.
\begin{proposition}\label{propos2}
Under the same assumption of Proposition~\ref{lemmafe}\,, for any $0 < \delta < 1$, there exists a subset of $Z_1$ of $Z^{m \times m}$
with measure at least $1-\delta/4$, such that for any $\forall {\bm z} \in Z_1$
\begin{equation}\label{bound7}
S_2(\bm z, \lambda) \leq
\frac{1}{2}D(\lambda)
+16(C_{\theta}+1)\bigg(B+M^*+\mathcal{G} \sqrt{\frac{D(\lambda)}{\lambda}}\bigg) \log\frac{4}{\delta}m^{-\frac{1}{2-\theta}} + \frac{1}{m} \Big( \mathcal{G} \sqrt{\frac{D(\lambda)}{\lambda}} +M^* \Big)\,.
\end{equation}
\end{proposition}

In the next, we aim to bound $S_1(\bm z, \lambda)$ with respect to the samples $\bm z$.
Thus a uniform concentration inequality for a family of functions containing $k^{(\varepsilon)}_{\bm z, \lambda}$ is needed to estimate $S_1$.
Since we have $k^{(\varepsilon)}_{\bm z, \lambda} \in \mathcal{B}_R$ by Eq.~\eqref{BRradius}, we shall bound $S_1$ by the following concentration inequality with a properly chosen $R$,  with proof deferred to Appendix~\ref{sec:s1}.
\begin{proposition}\label{propos1}
Under the same assumption of Proposition~\ref{lemmafe} and Eq.~\eqref{assumpN}, for any $0 < \delta < 1$, $R \geq 1$, $B > 0$, there exists a subset $Z_2$ of $Z^{m \times m}$
with measure at least $1-\delta/4$, such that for any $\bm z \in \mathscr{W}(R) \cap Z_2$,
\begin{equation}\label{boundS1}
 \begin{split}
 S_1(\bm z, \lambda) &= \Big\{ \mathcal{E}\big(\pi_B(k^{(\varepsilon)}_{\bm{z},\lambda})\big)\! - \! \mathcal{E}(k_{\rho}) \Big\} \!-\! \Big\{ \mathcal{E}_{\bm{z}}\big(\pi_B(k^{(\varepsilon)}_{\bm{z},\lambda})\big) \!- \! \mathcal{E}_{\bm{z}}(k_{\rho}) \Big\} \\
 & \leq 20 \epsilon^*(m,R,\frac{\delta}{4}) +  \frac{1}{2}\Big\{ \mathcal{E}\big(\pi_B(k^{(\varepsilon)}_{\bm{z},\lambda})\big) - \mathcal{E}(k_{\rho}) \Big\} + \frac{1}{m}(B+M^*) \,,
\end{split}
\end{equation}
where $\epsilon^*(m,R,\frac{\delta}{4})$ is given by
\begin{equation*}
  \epsilon^*(m,R,\frac{\delta}{4}) = 16(M^*+B)(C_{\theta}+1)\left( \log\frac{4}{\delta} m^{-\frac{1}{2-\theta}} + C_sR^sm^{-\frac{1}{2+s-\theta}}R^{\frac{s}{1+s}} \right)\,.
\end{equation*}
\end{proposition}

The left in the error decomposition demonstrated by Proposition~\ref{errdec} is to bound $\frac{1}{m^2} \sum_{i,j=1}^{m} | \pi_B(Y_{ij}) - Y_{ij} |$, which involves the unboundedness of the output $y$.
Following Proposition 5 in \citep{Shi2014Quantile}, under the assumption of Eq.~\eqref{Momenthypothesis}, for any $d \in \mathbb{N}$ and $0 < \delta < 1$, there exists a subset $Z_3$ of $Z^{m \times m}$ with measure at least $1-\delta/4$, such that
\begin{equation}\label{boundyy}
  \frac{1}{m^2} \sum_{i,j=1}^{m} | \pi_B(y_{ij}) - y_{ij} | \leq c \Big\{ (d+1)! + 2^{d+2} d^d \big\} M^{d+1} B^{-d} + \frac{12M 2^{d+1}}{m} \log \frac{4}{\delta}\quad \forall \bm z \in Z_3 \cap \mathscr{W}(R) \,,
\end{equation}
where we use $\frac{6M}{m-1} \leq \frac{12M}{m}$ for any $m > 1$.

\subsection{Derive Convergence Rates}
\label{sec:cr}
Based on above analyses, combining the bounds in Proposition~\ref{errdec}\,,~\ref{propos2}\,,~\ref{propos1}\,, Eq.~\eqref{flambound} and Eq.~\eqref{boundyy}\,, the excess error $ \mathcal{E}\big(\pi_B(k^{(\varepsilon)}_{\bm{z},\lambda})\big) - \mathcal{E}(k_{\rho})$ can be bounded by the following proposition, with the proof in Appendix~\ref{sec:excessfin}.
\begin{proposition}\label{proexcessfin}
Assume that $| k_{\rho}(\bm x, \bm x')| \leq M^*$ with $M^* \geq 1$.
$\rho$ has a median of $p$-average type $q$ with some $p\in (0,\infty]$ and $q\in (1,\infty)$ and satisfies assumptions Eq.~\eqref{theta} with $0 < \theta \leq 1$ , Eq.~\eqref{Dlambda} with $0 < r \leq 1$, and Eq.~\eqref{Momenthypothesis} with $c>0$.
Assume that for some $s>0$, take $\lambda=m^{-\alpha}$ with $ 0 < \alpha \leq 1$ and $\alpha < \frac{1+s}{s(2+s-\theta)}$. Set $\varepsilon = m^{-\gamma}$ with $\alpha r \leq \gamma \leq \infty$.
Then for $B \geq M^*$, $d \in \mathbb{N}$, and $0 < \delta < 1$ with confidence $1-\delta$, there holds
\begin{equation*}
\begin{split}
  \mathcal{E}\big(\pi_B(k^{(\varepsilon)}_{\bm{z},\lambda})\big) - \mathcal{E}(k_{\rho}) & \leq 4\widetilde{C} \log\frac{4}{\delta} m^{- \Theta}
    + 2c \Big\{ (d+1)! + 2^{d+2} d^d \Big\} M^{d+1} B^{-d} + \frac{12M 2^{d+1}}{m} \log \frac{4}{\delta}\,,
    \end{split}
\end{equation*}
where $\widetilde{C}$ is a constant independent of $m$ or $\delta$ and the power index $\Theta$ is
\begin{equation}\label{thetafin}
\begin{split}
  \Theta
 & = \min \left\{ {\alpha r},   \frac{1+\alpha r -\alpha}{2-\theta}
 , \frac{1}{2+s-\theta} - \frac{\alpha s}{1+s}  \right\}\,.
  \end{split}
\end{equation}
\end{proposition}

Now we are ready to give the proof of Theorem~\ref{theorem1s}\,.
\begin{proof}
Using Proposition~\ref{lemmafe} and~\ref{proexcessfin}\,, for any $B \geq M^*$ with confidence $1- \delta$, there holds
\begin{equation}\label{the1pre}
\begin{split}
  & \big\| \pi_{M^*}(k^{(\varepsilon)}_{\bm{z},\lambda})  - k_{\rho} \big\|_{L_{\rho_X}^{p^*}}  \leq  \big\| \pi_{B}(k^{(\varepsilon)}_{\bm{z},\lambda})  - k_{\rho} \big\|_{L_{\rho_X}^{p^*}}
    \leq C_q \left( 4\widetilde{C} \log\frac{4}{\delta} m^{- \Theta} \right)^{\frac{1}{q}}B \\
   & + B\left( 2c \Big[ (d+1)! + 2^{d+2} d^d \Big] M^{d+1} B^{-d}\right)^{1/q}  +\left(12M 2^{d+1} \log \frac{4}{\delta}\right)^{1/q} B^{1-1/q} m^{-1/q}\,,
  \end{split}
\end{equation}
where $\Theta$ is given by Eq.~\eqref{thetafin}, and the first inequality admits by $|k_{\rho}| \leq M^*$ almost surely.
Following \citep{Shi2014Quantile}, we have
\begin{equation}\label{bousde}
\left( 2c \Big[ (d+1)! + 2^{d+2} d^d \Big] M^{d+1} B^{-d}\right)^{1/q} \leq 2^{\frac{\Theta+2\epsilon}{q\epsilon}}(20cM)m^{-\Theta/q}\,,
\end{equation}
and
\begin{equation}\label{bousthr}
\left(12M 2^{d+1} \log \frac{4}{\delta}\right)^{1/q} B^{1-1/q} m^{-1/q} \leq 2^{\frac{d+1}{q}} \left(12M \log \frac{4}{\delta}\right)^{1/q} B m^{-1/q}\,.
\end{equation}
Finally, we complete the proof by combining Eqs.~\eqref{bousde} and \eqref{bousthr} into Eq.~\eqref{the1pre}
\begin{equation*}
\begin{split}
\big\| \pi_{M^*}(k^{(\varepsilon)}_{\bm{z},\lambda})  - k_{\rho} \big\|_{L_{\rho_X}^{p^*}}  \leq  \widetilde{C}^{\epsilon}_{\underline{X}, \rho, \alpha, \gamma} \left( \log\frac{4}{\delta} \right)^{1/q} m^{\epsilon - \frac{\Theta}{q}} \,,
\end{split}
\end{equation*}
with
\begin{equation*}
  \widetilde{C}^{\epsilon}_{\underline{X}, \alpha} = 3(M + M^*) \max \Big\{ C_q ( 4\widetilde{C})^{\frac{1}{q}}, (20cM)^{1/q}, (12M)^{1/q} \Big\} 2^{\frac{\Theta+2 \epsilon}{q\epsilon}} \Theta \epsilon^{-1}\,,
\end{equation*}
which concludes the proof.
\end{proof}

\subsection{Theoretical Results on Kernel Approximation}


A series of kernel approximation schemes, e.g., divide-and-conquer \citep{zhang2013divide}, distributed learning \citep{lin2017distributed}, Nystr\"{o}m approximation \citep{rudi2015less}, random features \citep{Rudi2017Generalization}, have been extensively studied in learning theory, mainly on kernel ridge regression in RKHS. Recently, much efforts focus on the combination of several strategies, e.g., divide-and-conquer with Nystr\"{o}m approximation \citep{yin2020divide}, distributed learning with stochastic gradient descent (SGD) \citep{lin2020optimal}, random features with SGD \citep{carratino2018learning}. 
Accordingly, following \citep{yin2020divide,rudi2017falkon}, our derived theoretical result on the full problem can be extended to the approximation version with divide-and-conquer and Nystr\"{o}m approximation. Here we briefly present the error decomposition result of KRR in hyper-RKHS under such two kernel approximation settings and sketch our key ideas. 

Define the noise-free version of $\widetilde{k}_{M, \mathcal{V}_c, \lambda}$ by Nystr\"{o}m approximation on the subset $\mathcal{V}_c$ as
\begin{equation*}
	\begin{split}
		& \widetilde{k}_{M, \mathcal{V}_c, \rho, \lambda} (\bm x, \bm x')  = \sum_{i,j=1}^M \! \widetilde{\beta}_{ij} \underline{k} \big((\bm x, \bm x'), (\widetilde{\bm x}_i, \widetilde{\bm x}_j) \big) \,, \\
		& ~~\mbox{with}~~ \mathop{\mathrm{vec}} (\widetilde{\bm \beta}) \!=\! \!\Big(\underline{\bm K}_{nM}^{\!\top}\underline{\bm K}_{nM} + \lambda n^2 \underline{\bm K}_{MM} \Big)^{-1} \underline{\bm K}_{nM}^{\!\top} \mathop{\mathrm{vec}}(k_{\rho}(\bm x_r, \bm x_s))~~ r,s \in \{ 1,2, \cdots, n \}\,,
	\end{split} 
\end{equation*}
and further $\bar{k}_{M, \mathcal{V}, \rho, \lambda} = 1/v \sum_{c=1}^v \tilde{k}_{M, \mathcal{V}_c, \rho, \lambda}$ is the average of all the $v$ partitions.
Define the noise version of the local estimator (without Nystr\"{o}m approximation) on the subset $\mathcal{V}_c$ as
\begin{equation*}
	\begin{split}
		& \widetilde{k}_{\mathcal{V}_c, \rho, \lambda} (\bm x, \bm x')  = \sum_{i,j=1}^n \! \widetilde{\beta}_{ij} \underline{k} \big((\bm x, \bm x'), ({\bm x}_i, {\bm x}_j) \big) \,, \\
		& ~~\mbox{with}~~ \mathop{\mathrm{vec}} (\widetilde{\bm \beta}) \!=\! \!\Big( \underline{\bm K}_{nn} + \lambda n^2 \bm I \Big)^{-1} \mathop{\mathrm{vec}}(k_{\rho}(\bm x_r, \bm x_s))~~ r,s \in \{ 1,2, \cdots, n \}\,,
	\end{split} 
\end{equation*}
and further $\bar{k}_{\mathcal{V}, \rho, \lambda} = 1/v \sum_{c=1}^v \tilde{k}_{\mathcal{V}_c, \rho, \lambda}$ is the average of all the $v$ partitions.
Then the error decomposition for KRR in hyper-RKHS under such two approximation strategies can be similarly obtained by\footnote{The notation $a(v,M,m) \lesssim b(v,M,m)$ means that $a(v,M,m) \leq C b(v,M,m)$ where $C$ is some absolute constant independent of $v,M,m$.}
\begin{equation*}
	\begin{split}
		\mathbb{E} \mathcal{E}(\bar{k}_{M,\lambda}) - \mathcal{E}(k_{\rho})  \lesssim &~  \frac{1}{v} \mathbb{E} \| \widetilde{k}_{M, \mathcal{V}_c, \lambda} - \widetilde{k}_{M, \mathcal{V}_c, \rho, \lambda} \|^2_{L^2_{\rho_{X}}} +  \underbrace{ \mathbb{E}\| \bar{k}_{M, \mathcal{V}, \rho, \lambda} - \bar{k}_{ \mathcal{V}, \rho, \lambda} \|^2_{L^2_{\rho_{X}}} }_{\mbox{Nystr\"{o}m error}} \\ & + \frac{1}{v} \mathbb{E} \| \bar{k}_{ \mathcal{V}, \rho, \lambda} - k_{\lambda} \|^2_{L^2_{\rho_{X}}}   + \underbrace{ \| k_{\lambda} - k_{\rho} \|^2_{L^2_{\rho_{X}}} }_{\mbox{approximation error}},~~~ c = 1,2, \dots, v\,,
	\end{split}
\end{equation*} 
where the first term and the third term are sample error which controls the variance of the outputs $y$ and sample variance, the second term involves with Nytr\"{o}m approximation and the last term is the bias, i.e., the approximation error. 

In particular, the approximation error can be directly upper bounded by Eq.~\eqref{Dlambda}. The key part in the analysis is to use the Bernstein's inequality to study the relationship between the empirical pair sample and its expectation, which has been established in Lemma~\ref{lemma4}. Accordingly, proofs for sample error and Nytr\"{o}m error can be exactly obtained by combining Lemma~\ref{lemma4} with previous results \citep{yin2020divide,rudi2017falkon}. We therefore omit the proof in this paper.

\section{Experiments}
\label{experiment}

We evaluate the proposed two regression models with squared loss and the $\varepsilon$-insensitive loss in hyper-RKHS, termed as ``hyper-KRR" and ``hyper-SVR" for learning kernels, and then apply them to classification tasks.
First, we experimentally investigate the approximation performance of our methods for the known kernels on the UCI repository\footnote{\scriptsize{\url{https://archive.ics.uci.edu/ml/datasets.html}}}.
Second, we conduct experiments to learn an underlying kernel from the ``ideal'' kernel on a wide range of classification problems on the UCI classification datasets.
Third, for scalability, we test our methods on two large datasets including \emph{ijcnn1} and \emph{covtype}\footnote{\scriptsize{Both data sets are available at
\url{https://www.csie.ntu.edu.tw/~cjlin/libsvmtools/datasets/}}
}.
Last, for out-of-sample extensions, we apply our method to non-parametric kernel learning on the \emph{MNIST} handwritten digits dataset \citep{L1998Gradient}.
The experiments implemented in MATLAB are conducted on a PC with Intel$^\circledR$ i7-8700K CPU (3.70 GHz) and 64 GB RAM.
The source code of our implementation can be found in \url{http://www.lfhsgre.org}.

During training, ${\sigma}^2$ in the Gaussian hyper-kernel is set to the variance of data, and $\sigma_h^2$ is tuned via
5-fold cross validation over the values $\{ 0.25 {\sigma}^2, 0.5 {\sigma}^2, {\sigma}^2, 2 {\sigma}^2, 4 {\sigma}^2\}$.
The regularization parameters $\lambda$ in KRR and $C$ in SVR are searched on grids of $\log_{10}$ scale in the range of $[10^{-5},10^5]$.
The two slack variables $\hat{\xi}_{ij}, \check{\xi}_{ij}$ in SVR are set to $0.1$ and $0.01$, respectively.

\subsection{Approximating known PD/non-PD kernels}
\label{sec:apppdnon}
Here we carry out experiments to investigate the approximation performance for known kernels.
The out-of-sample extension based algorithm \citep{Pan2016Out} is taken into comparisons.
This method solves a nonnegative least squares in hyper-RKHS, which can be regarded a special case of hyper-KRR.
Nevertheless, we do not want to claim that the learned (indefinite) kernel in our framework is better than the PD one from \citep{Pan2016Out}.
Instead, our target is to show the utility or flexibility of our framework.
For fair comparison, these three algorithms in hyper-RKHS are associated with the same hyper-kernel, i.e., the hyper-Gaussian kernel used in this subsection.

For the experiments on UCI data sets, the data points are partitioned into 40\% labeled data, 40\% unlabeled data, and 20\% test data.
The labeled and unlabeled data points form the training dataset.
Such setting follows with \citep{Pan2016Out}, which simultaneously considers tranductive learning and inductive learning.
Here the pre-given kernel matrix is generated by a known kernel including a positive definite one and an indefinite one.
Learning on known kernels focuses on the approximation performance of the compared algorithms on these kernels.
The used evaluation metric here is relative mean square error (RMSE) between the learned regression function $k^*(\bm x, \bm x')$ and the pre-given kernel matrix $\bm K$ over $m^2$ pairwise data points.
Besides, we also evaluate our kernel learning methods incorporated into SVM for classification.
As a consequence, such experimental setting on known kernels help us to comprehensively investigate the approximation ability of the compared algorithms on PD or non-PD kernels.

{\bf Results on the known Gaussian kernel:} Here the pre-given kernel matrix is generated by a known Gaussian kernel.
Table~\ref{Tabinde} reports the experimental results in terms of classification accuracy and test root mean squared error (RMSE) for out-of-sample extensions on the known Gaussian kernel.
From the results, we can see that  the proposed hyper-SVR and hyper-KRR have the capability of approximating the Gaussian kernel.
Further, in terms of the test accuracy, hyper-SVR performs better than the other two methods on test data.
But, regarding to hyper-KRR and hyper-SVR, in general, we see that classification performance and approximation accuracy are not well correlated.
A better approximation quality cannot guarantee better classification performance.
This is not a unique phenomenon in our algorithm but a common issue in the kernel approximation topic \citep{avron2017random,munkhoeva2018quadrature,liu2020survey}.
Approximation and generalization appears two correlated tasks but experimentally not. How to bridge the gap between good (distinct) approximation and indistinctive generalization performance still remains an open question in theory.

\begin{table*}[t]
  \centering
  \scriptsize
  \begin{threeparttable}
  \caption{RMSE performance on test data and classification accuracy of (mean$\pm$std. deviation) of each compared algorithm on unlabeled data and test data, in which the given kernel matrix is generated by a Gaussian kernel and a $\log$ kernel, respectively. The best performance is highlighted in \textbf{bold}. The results directly achieved by these three known kernels do not participate in ranking.}
  \label{Tabinde}
    \begin{tabular}{cccccccccccccccccccc}
    \toprule
    &{Dataset}&\multirow{2}{0.35cm}{type} &fertility&australian &wine
    &sonar&heart&guide1-t\cr
    \cmidrule(lr){4-9}
    &(\#data, \#feature) &&(100, 9) &(690, 14) &(178, 13) &(208, 60) &(270, 13) &(4000, 4)\cr
    \midrule
    \multirow{11}{0.2cm}{\rotatebox{90}{Gaussian kernel}} &\multirow{2}{1.8cm}{\centering{the \emph{known} Gaussian kernel}} &unlabel &95.02$\pm$7.11 &83.92$\pm$1.90 &96.68$\pm$1.93 &75.34$\pm$3.07 &78.63$\pm$4.52&95.67$\pm$3.71\\
    &&test  &87.45$\pm$6.77 &82.64$\pm$2.23&98.02$\pm$1.81&72.14$\pm$1.38 &80.91$\pm$4.89&95.11$\pm$3.33\\
    \cmidrule(lr){3-9}
    &\multirow{3}{1.8cm}{\citep{Pan2016Out}} &unlabel & 88.03$\pm$8.11 &81.56$\pm$3.78 &96.89$\pm$1.4 1&73.78$\pm$5.67 &{\bf 81.45}$\pm$3.56&{\bf 95.89}$\pm$2.83\\
    &&test &85.51$\pm$7.72 &82.64$\pm$3.60 &96.33$\pm$2.67 &73.83$\pm$8.02  &79.04$\pm$5.11 &93.42$\pm$2.81\\
    &&RMSE  &0.152 &{\bf 0.088} &0.123 &0.165  &0.128 &{0.288} \cr
    \cmidrule(lr){3-9}
    &\multirow{3}{1.5cm}{hyper-KRR}  &unlabel  &92.53$\pm$7.80 &81.52$\pm$4.01 &{\bf 97.33}$\pm$1.45 &{\bf 77.12}$\pm$4.21 &80.14$\pm$5.02 &95.81$\pm$2.67\\
    &&test &86.01$\pm$9.33 &{\bf 83.63}$\pm$4.20 &96.62$\pm$3.41 &67.82$\pm$8.20 &{\bf 82.73}$\pm$5.22 &95.65$\pm$2.54\\
    & &RMSE   &{\bf 0.081} &0.138 &{\bf 0.062} &{\bf 0.085} &{\bf 0.095}&{\bf 0.104} \cr
    \cmidrule(lr){3-9}
    &\multirow{3}{1.8cm}{hyper-SVR}  &unlabel  &{\bf 95.64}$\pm$8.52 &{\bf 82.24}$\pm$3.70 &97.13$\pm$1.20 &73.24$\pm$6.13 &81.30$\pm$3.62 &95.14$\pm$3.33\\
    &&test &{\bf 88.53}$\pm$7.10 &82.61$\pm$3.93 &{\bf 98.04}$\pm$2.22 &{\bf 75.54}$\pm$6.42 &80.12$\pm$3.20 &{\bf 97.92}$\pm$2.32\\
    & &RMSE  &0.102 &0.108 &0.089 &0.143 &0.120 &0.120 \cr
        \hline
  \multirow{9}{0.2cm}{\rotatebox{90}{log kernel}}&  \multirow{2}{1.2cm}{the \emph{known} log kernel}  &unlabel &95.23$\pm$5.72 &84.04$\pm$1.93 &98.04$\pm$1.32&74.44$\pm$7.32 &80.51$\pm$1.42 &95.24$\pm$4.11\\
  & &test &81.52$\pm$8.81 &83.90$\pm$1.92 &96.12$\pm$2.62 &78.33$\pm$6.82 &81.44$\pm$6.50 &95.02$\pm$3.71 \cr
 \cmidrule(lr){3-9}
  &  \multirow{1}{1.5cm}{\centering{\citep{Pan2016Out}\tnote{1}}}  &- &- &- &- &- &- &-  \cr
 \cmidrule(lr){3-9}
  &  \multirow{3}{1.5cm}{hyper-KRR}   &unlabel  &{\bf 98.01}$\pm$2.62 &73.84$\pm$1.92 &{\bf 97.31}$\pm$1.33 &72.34$\pm$5.32 &81.02$\pm$2.21 &{95.12}$\pm$1.43\\
   & &test &88.02$\pm$4.81 &76.64$\pm$5.14 &60.13$\pm$8.08 &65.84$\pm$9.56 &77.63$\pm$7.02 &91.84$\pm$3.93\\
   &   &RMSE   &0.005 &0.717 &0.748 &0.697 &0.435 &0.827 \cr
\cmidrule(lr){3-9}
   & \multirow{3}{1.5cm}{hyper-SVR}   &unlabel &97.83$\pm$4.22 &{\bf 84.12}$\pm$1.73 &97.02$\pm$1.91 &{\bf 72.72}$\pm$4.67  &{\bf 81.74}$\pm$3.12 &{\bf 96.84}$\pm$1.34\\
  &  &test &{\bf 93.53}$\pm$2.44 &{\bf 83.42}$\pm$3.80 &{\bf 95.53}$\pm$5.12 &{\bf 66.74}$\pm$6.32 &{\bf 80.93}$\pm$4.70 &{\bf 94.22}$\pm$1.63\\
     & &RMSE  &{\bf 0.002} &{\bf 0.474} &{\bf 0.138} &{\bf 0.174} &{\bf 0.196} &{\bf 0.523} \cr
    \bottomrule
    \end{tabular}
    \begin{tablenotes}
        \footnotesize
        \item[1] We omit the results provided by \cite{Pan2016Out} on the $\log$ kernel because this method cannot output a nonnegative coefficient vector for approximation due to the negative values in the $\log$ kernel.
\end{tablenotes}
    \end{threeparttable}
    \vspace{-0cm}
\end{table*}

{\bf Results on the known $\log$ kernel:}
Here we conduct experiments on a known indefinite kernel to generate the pre-given output. 
The used $\log$ kernel \citep{boughorbel2005conditionally} is given by $k(\bm x,\bm x') = - \log(1+\frac{\| \bm x - \bm x'\|}{{\sigma}})$ with the chosen $\sigma=1$.
Table~\ref{Tabinde} reports the classification accuracy and test root mean squared error for out-of-sample extensions across the $\log$ kernel.
In terms of the approximation ability of these compared algorithms, hyper-SVR performs best to approximate these two indefinite kernels.
It can be noticed that, the algorithm in \citep{Pan2016Out} is not able to approximate the $\log$ kernel due to its negative values, and thus is infeasible for such kernel.
In general, we see improvement yielded by our two regression models on the final classification accuracy, which also shows the flexibility of indefinite kernel learning models.

\subsection{Learning by approximating the ``ideal" kernel}
\label{sec:appideal}
As aforementioned, the ``ideal" kernel can be used to guide the kernel learning task.
Here we evaluate our methods with other representative kernel learning based algorithms embedded in SVM for classification.

{\bf Experimental settings:} Table~\ref{Tabideal} lists a brief description of six UCI datasets including the number of data $n$ and the feature dimension $d$.
The data are normalized to $[0,1]^d$ in advance. The compared algorithms include
\begin{itemize}
  \item KTA \citep{Cortes2012Algorithms}: A two-stage kernel learning framework jointly learns the weights of base kernels by maximizing the alignment with the ``ideal" kernel in stage 1, and then is embedded into SVM for classification. 
Here the base kernels are chosen as eleven Gaussian kernels with the kernel width $\sigma \in \{ 2^{-5}, 2^{-4}, \cdots, 2^5 \}$.
  \item BMKL \citep{Gonen2012Bayesian}: A Bayesian multiple kernel learning algorithm ensemble eleven Gaussian kernels with different kernel widths $\sigma \in \{ 2^{-5}, 2^{-4}, \cdots, 2^5 \}$ and three polynomial kernels with degrees $1,2,3$. 
  \item RF \citep{AmanNIPS2016}: A nonparametric kernel learning framework generates a large number of random features (we set to 10,000 in our experiment) by the Gaussian kernel, and then learn their weights based on target alignment.
  \item MIKL \citep{kowalski2009multiple}: A multiple indefinite kernel learning framework ensembles a linear kernel and two Gaussian kernels with $\sigma = 0.1$ and $\sigma = 100$ via a mixed norm regularization scheme. The combination coefficient can be negative, which allows for indefinite kernel learning. In our experiments, we use the $\ell_1$-norm regularization as an example for comparison.
  \item SVM-CV: The SVM classifier with the Gaussian kernel is served as a baseline, where the balance parameter $C$ and the kernel width $\sigma$ are tuned by 5-fold cross validation on a grid of points, i.e., $\sigma=[2^{-5}, 2^{-4},\dots,2^5]$ and $C = [2^{-5}, 2^{-4},\dots,2^5]$.
\end{itemize}
Our methods includes four version determined by the used two regressors: KRR and SVR in hyper-RKHS and the used two hyper-kernels: hyper-Gaussian kernel and hyper-Wishart kernel.
We follow with the setting in Section~\ref{sec:apppdnon}, these kernel learning based algorithms are conducted by randomly picking 40\% of the data for training and the rest for test. The experiments are repeated 10 trials on these six datasets.

\begin{table*}
  \centering
  \footnotesize
  \begin{threeparttable}
  \caption{Classification accuracy of (mean$\pm$std. deviation) of our algorithms on test data for the ``ideal" kernel versus representative kernel learning based approaches equipped with various base kernels. The best performance is highlighted in \textbf{bold}.}
  \label{Tabideal}
    \begin{tabular}{ccccccccccccccccccc}
    \toprule
     Dataset &fertility&australian &wine &sonar
    &heart&guide1-t\cr
    \cmidrule(lr){2-7}
    (\#data, \#feature) &(100, 9) &(690, 14) &(178, 13)  &(208,60) &(270, 13) &(4000, 4)\cr
    \midrule
    KTA
    &86.67$\pm$2.05 &82.52$\pm$1.55 &96.20$\pm$2.22 &80.41$\pm$2.88 &83.24$\pm$1.62 &88.22$\pm$0.68\cr
    \hline
    BMKL
    &85.50$\pm$3.34 &84.61$\pm$1.56 &95.48$\pm$2.21 &81.68$\pm$2.93 &{\bf 84.46}$\pm$1.88 &96.05$\pm$0.46 \cr
    \hline
    RF
    &81.33$\pm$5.31 &82.17$\pm$2.02 &94.88$\pm$2.94 &80.52$\pm$3.58 &82.46$\pm$2.12 &95.83$\pm$0.48 \cr
    \hline
    MIKL
    &86.83$\pm$1.46 &{\bf 86.45}$\pm$0.98 &94.78$\pm$2.02 &77.72$\pm$5.57 &82.90$\pm$2.21 &89.74$\pm$0.57 \cr
    \hline
    SVM-CV
    &86.50$\pm$2.53 &85.14$\pm$0.84 &95.14$\pm$2.32 &80.40$\pm$4.48 &80.43$\pm$3.35 &96.49$\pm$0.44 \cr
    \hline
    \hline
    hyper-KRR(Gaussian)
     &88.50$\pm$3.37 & 82.82$\pm$4.80 &92.53$\pm$4.83 &77.67$\pm$5.07 &81.11$\pm$3.12 &{93.92}$\pm$3.44\cr
   \hline
   hyper-SVR(Gaussian)
   &89.50$\pm$5.50 &85.68$\pm$3.56 &{\bf 97.32}$\pm$1.82 &{\bf 82.32}$\pm$3.34 &82.65$\pm$2.35 &{93.22$\pm$2.88}\cr
    \hline
    hyper-KRR(Wishart)
   &90.25$\pm$2.34 &81.21$\pm$2.84 &94.83$\pm$2.77 &78.45$\pm$2.24 &80.15$\pm$2.71 &{92.65$\pm$2.58}\cr
   \hline
   hyper-SVR(Wishart)
   &{\bf 90.30}$\pm$2.12 &84.23$\pm$2.27 &96.65$\pm$2.35  &81.13$\pm$2.94 &81.43$\pm$2.52 &{\bf 96.52$\pm$1.48}\cr
    \bottomrule
    \end{tabular}
    \end{threeparttable}
\end{table*}

{\bf Experimental results:}
Table~\ref{Tabideal} reports the test classification accuracy of all compared methods.
Compared with KTA and RF based on kernel target alignment, our methods perform better to learn the underlying kernel from the ``ideal" kernel in hyper-RKHS, and thus achieve promising performance with noticeable margins.
When compared to BMKL and MIKL based on multiple kernel learning, the proposed SVR with Gaussian/Wishart hyper-kernel performs well on several datasets, which verifies the effectiveness of our kernel learning scheme.
It indicates that the learned underlying kernel is flexible beyond a linear combination of several base kernels.
For self comparisons, in terms of the test accuracy, the proposed hyper-SVR with Gaussian/Wishart kernel is superior to the remaining three versions as a whole.
Regarding to the loss function, our regression model with the squared loss in hyper-RKHS is often inferior to that with the $\varepsilon$-insensitive loss whatever the hyper-kernel is chosen.

\begin{figure*}[t]
	\centering
	\subfigure{\label{fignum}
		\includegraphics[width=0.315\textwidth]{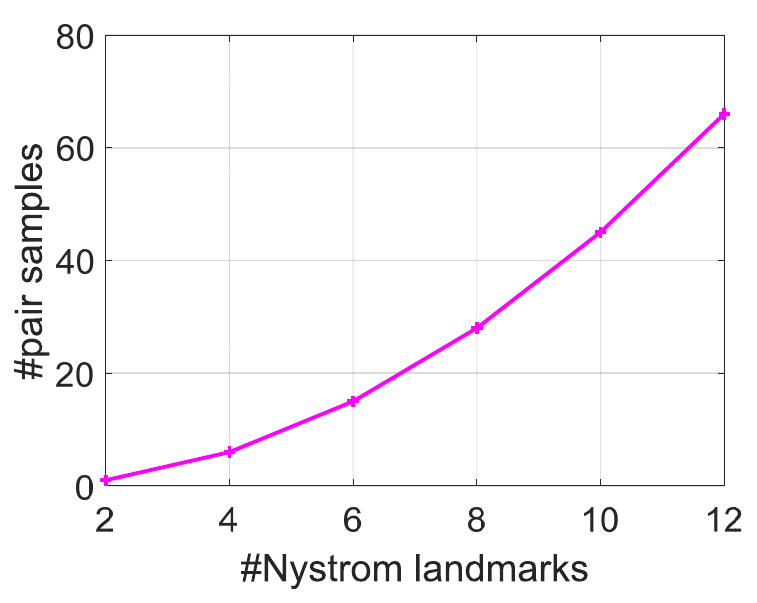}}
	\subfigure{\label{figapp}
		\includegraphics[width=0.317\textwidth]{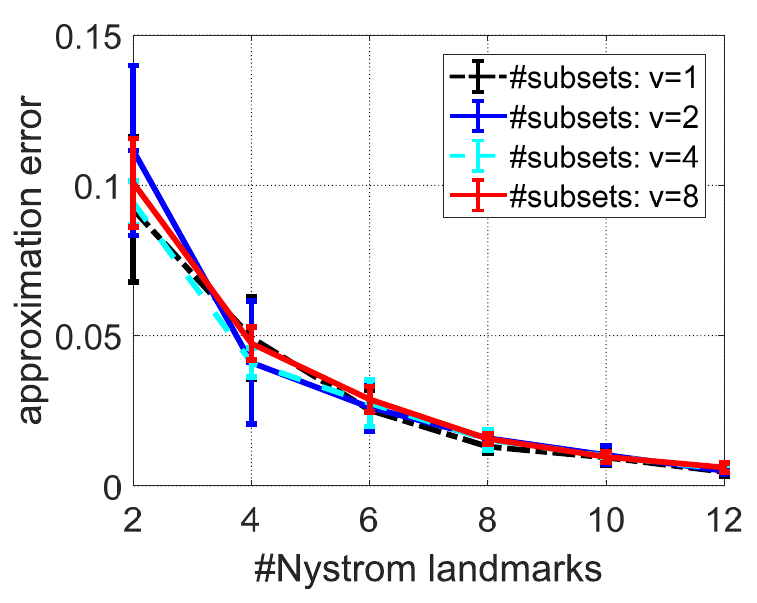}}
	\subfigure{\label{figtime}
		\includegraphics[width=0.315\textwidth]{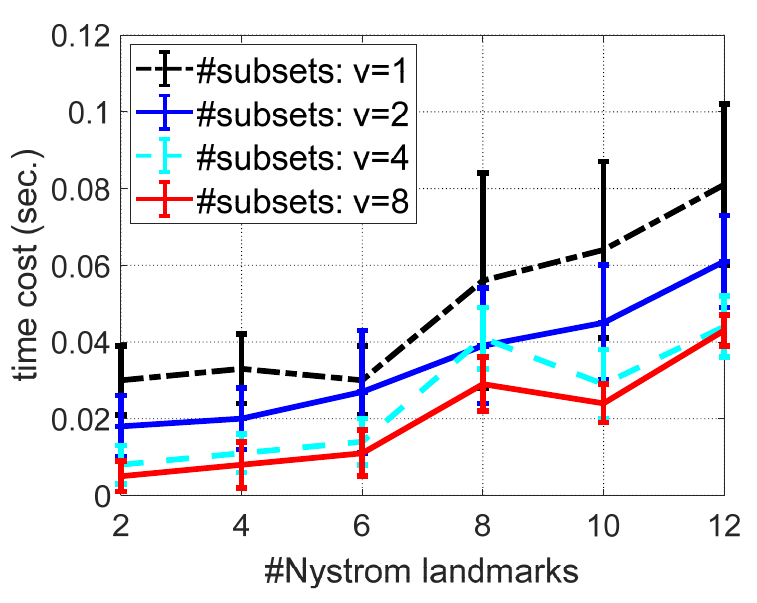}}
	\caption{Approximating the Gaussian hyper-kernel matrix with varying \#Nystr\"{o}m landmarks $M$ and \#subsets $v$ on the \emph{heart} dataset.}	\label{fignys}
	\vspace{-0.05cm}
\end{figure*}

\subsection{Validation of Kernel Approximation and Results on Large Scale Datasets}
\label{sec:largenys}
In this subsection, we first quantitatively evaluate the Gaussian hyper-kernel approximation effect by the number of Nystr\"{o}m landmarks $M$ and subsets $v$, and then apply such two schemes to large scale datasets.

In our experiment, we equally divide the $m$ training data into $v$ partitions $\{ \mathcal{V}_1, \mathcal{V}_2, \cdots, \mathcal{V}_v \}$.
Then on the subset $\mathcal{V}_c$ ($c=1,2,\dots,v$), we use $M$ landmarks for Nystr\"{o}m approximation to obtain an approximated hyper-kernel matrix $\widetilde{\underline{\bm K}}^{(c)} = \underline{\bm K}_{{n}{M}} \underline{\bm K}_{{M}{M}}^{\dagger} \underline{\bm K}_{{n}{M}}^{\!\top}$.
Finally, the approximated hyper-kernel matrix averaged on $v$ subsets is given by $\widetilde{\underline{\bm K}} = 1/v \sum_{c=1}^v \widetilde{\underline{\bm K}}^{(c)}$ such that $\underline{\bm K} \approx \widetilde{\underline{\bm K}}$.
In this case, the approximation error is evaluated by $\| \widetilde{\underline{\bm K}} - \underline{\bm K} \|_{\mathrm{F}}/\| \underline{\bm K} \|_{\mathrm{F}}$ on the $m^2$ pair samples.
Strictly speaking, the number of unduplicated pair samples is $\left(\begin{array}{c}
	2 \\
	m
\end{array}\right)$.

Since the number of pair samples dramatically increases with $m$, we consider a small-scale \emph{heart} dataset with $m=108$ training data for evaluation. 
If we take $M$ Nystr\"{o}m landmarks, the number of unduplicated pair samples can be reduced to $ \widetilde{M} = \left(\begin{array}{c}
	2 \\
M
\end{array}\right)$.
Figure~\ref{fignys} shows the number of unduplicated pair samples by $M$ landmarks, approximation error, and time cost for approximation (mean $\pm$ std. across 10 trials) on the \emph{heart} dataset.
We take $M=2,4,6,8,10,12$ and $v=1,2,4,8$ into comparison.
The approximation error under different number of landmarks $M$ and subsets $v$ is shown in
Figure~\ref{figapp}.
We find that, as the number of landmarks $M$ increases, the approximation error dramatically decreases even if $M$ is much smaller than training data size $m$. 
Nevertheless, the divide-and-conquer scheme, e.g., $v=2,4,8$, does not incur in extra approximation error when compared to the original case with $v=1$, which demonstrates its utility.
Specifically, this scheme is able to decrease time cost for kernel approximation as shown in Figure~\ref{figtime}, which validates its effectiveness in terms of computational efficiency.

After quantitatively evaluating the performance of the developed kernel approximation scheme (divide-and-conquer and Nystr\"{o}m approximation), we incorporate them into the studied model in hyper-RKHS on large scale datasets for prediction.
Here we choose two large scale datasets including \emph{ijcnn1} and \emph{covtype} to test the compared algorithms in hyper-RKHS on the ideal kernel.
Instead, for these three learning algorithms in hpyer-RKHS, we divide the data into $v$ disjoint subsets $\{ \mathcal{V}_1, \mathcal{V}_2, \cdots, \mathcal{V}_v \}$, and then conduct Nystr\"{o}m approximation on each subset.
The number of subsets is set to $v=5,10,20$ on the \emph{ijcnn1} dataset, and $v=50,100,200$ on the \emph{covtype} dataset.
Following \citep{Pan2016Out}, the number of Nystr\"{o}m landmarks is set to $M = 0.05m$.
Besides, we also include BMKL equipped with Gaussian kernels and polynomial kernels for comparison.
Note that, Nystr\"{o}m approximation on BMKL appears non-trivial, and thus we just incorporate BMKL into the divide-and-conquer framework and cooperate without the rankings.

Table~\ref{Tablarge} reports the number of training and test data, the number of subsets, the mean classification accuracy and the total time cost.
Experimental results show that all the three methods in hyper-RKHS can be feasible to large scale case, owing much to the developed kernel approximation scheme.
We find that, these three algorithms achieve similar performance in terms of classification accuracy and time cost.
As the number of subsets increases, these three algorithms achieve slight fluctuation on the test accuracy but significantly improve the computational efficiency.
Besides, BMKL achieves the best performance on classification accuracy but takes much more time cost for kernel learning.
Here we just report its results but do not include it for fair comparison as distributed BMKL is just equipped with the divide-and-conquer scheme without Nystr\"{o}m approximation.


\begin{table*}[t]
\renewcommand\arraystretch{1}
  \centering
  \small
  \begin{threeparttable}
  \caption{Classification accuracy and total time cost of each compared algorithm on two large scale datasets for the ideal kernel.}
  \label{Tablarge}
    \begin{tabular}{p{2cm}cccc|cccc}
    \toprule
    \centering{Dataset}&\centering{$v$} &\centering{\citep{Pan2016Out}} &\centering{hyper-KRR} &\centering{hyper-SVR} & distributed BMKL \cr
    \midrule
    \multirow{3}{0.8cm}{\centering{\emph{ijcnn1}}\\\#train=49,990\\ \#test=91,701} &\centering{5} &90.49\%(1354.2s)  &90.72\%(1375.2s)  &90.22\%(1322.4s) & 97.35\%(230576s) \\
     &\centering{10}  &89.72\%(835.8s) &89.71\%(846.8s)  &89.37\%(1156.1s) & 97.36\%(12020s) \\
     &\centering{20} & 90.49\%(743.5s) &90.94\%(752.1s)  &90.97\%(1035.9s) & 97.34\%(5844.8s) \\
     \cmidrule(lr){1-6}
    \multirow{3}{0.8cm}{\centering{\emph{covtype}}\\ \#train=232,405\\ \#test=232,405} &\centering{50} &68.32\%(4231.5s)  &70.41\%(4276.3s) &70.64\%(4353.5s) & 77.03\%(185182s)  \\
   &\centering{100} &69.67\%(3213.4s) &69.82\%(3241.5s)  &76.58\%(3352.2s) & 76.30\%(81551s)  \\
    &\centering{200}  &70.61\%(2300.6s) &70.45\%(2305.8s) &70.64\%(2317.5s) & 76.83\%(18001s) \cr
    \bottomrule
    \end{tabular}
    \end{threeparttable}
    \vspace{-0.2cm}
\end{table*}

\begin{figure*}[t]
	\centering
	\subfigure[digits 1 \emph{vs.} 9]{\label{fig1v9}
		\includegraphics[width=0.315\textwidth]{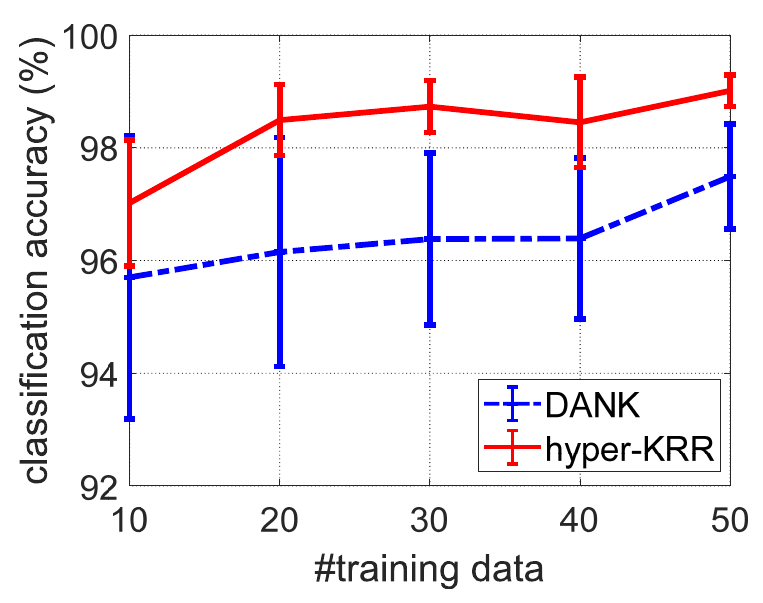}}
	\subfigure[digits 2 \emph{vs.} 7]{\label{fig2v7}
		\includegraphics[width=0.317\textwidth]{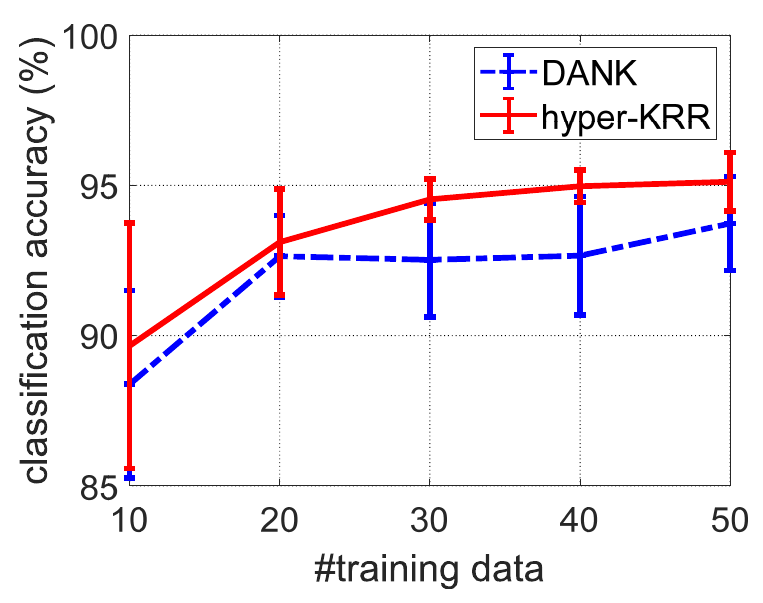}}
	\subfigure[digits 4 \emph{vs.} 6]{\label{fig3v6}
		\includegraphics[width=0.315\textwidth]{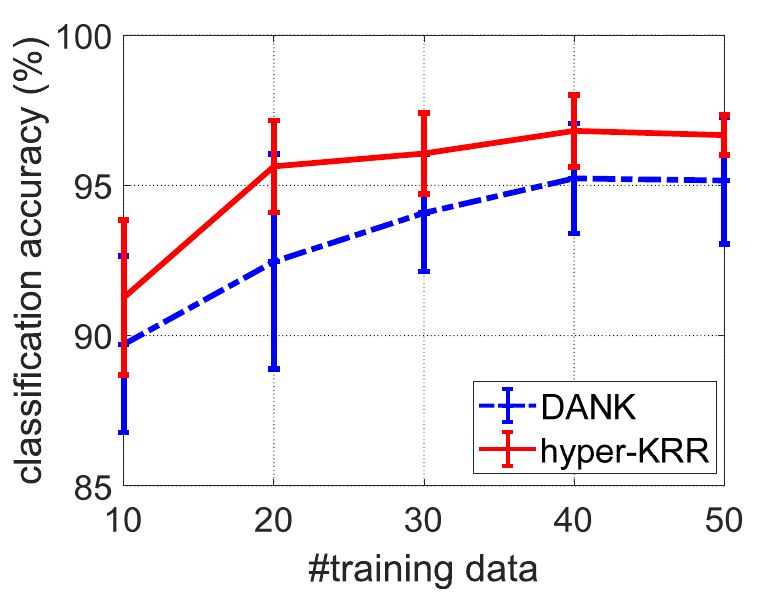}}
	\caption{Classification accuracy of DANK and hyper-KRR with varying number of training data on the \emph{MNIST} dataset.}	\label{figmnist}
	\vspace{-0.05cm}
\end{figure*}

\subsection{Out-of-sample extensions for nonparametric kernel learning}
\label{sec:onek}
As mentioned in the introduction, nonparametric kernel learning in a data-driven manner is often faced with the out-of-sample extensions issue, i.e., a non-parametric kernel/similarity matrix is learned but the underlying kernel function is unknown.
	For example, \cite{liu2020learning} propose a data-adaptive non-parametric kernel (DANK) learning framework to improves the model flexibility. In DANK, a data-adaptive matrix is learned based on the training data but is unknown on test data.
	To address the out-of-sample extension issues, they directly choose a simple reciprocal nearest neighbor scheme that extends the data-adaptive matrix from training to test data.
	We have to be faced with the inconsistency when using such interpolation scheme.
	Since the out-of-sample extension issue can be addressed by learning a kernel function in hyper-RKHS, here we compare hyper-KRR with the simple interpolation strategy on the \emph{MNIST} handwritten digits dataset \citep{L1998Gradient} for evaluation.
	
	In our experimental setting, several (easily confused) digit pairs, including 1 \emph{vs.} 9, 2 \emph{vs.} 7, and 4 \emph{vs.} 6, are taken into comparison.
	Specifically, we choose a few number of training data (i.e., 10, 20, 30, 40, 50) to validate the effectiveness of the employed out-of-sample extension strategy.
	We first use DANK to learn a non-parametric kernel matrix on the training data, then adopt the out-of-sample extension strategies to extend the kernel matrix from training data to test data, including the original reciprocal nearest neighbor scheme and the studied hyper-KRR, and finally incorporate it into SVM for classification on the test data. 
	Figure~\ref{figmnist} shows the test accuracy across 10 trials (mean$\pm$std. deviation) of the compared two out-of-sample extensions schemes.
	Results on classification accuracy indicate that the developed framework in hpyer-RKHS is able to achieve improvement with about 2\% margin on DANK when compared to the original reciprocal nearest neighbor scheme \citep{liu2020learning}.
	Such improvement on these three digit pairs demonstrates the effectiveness of the studied hyper-RKHS based algorithms for out-of-sample extension, especially when the training data size is small or limited.

\section{Discussion}
\label{sec:diss}
Here we briefly discuss the related topics on kernel learning and neural networks close to the studied framework in this paper.


Our kernel learning framework belongs to a \emph{two-stage} process that first learns a suitable kernel from the training data, and then uses the learned kernel in a conventional kernel machine, such as SVM or SVR for prediction.
One representative approach is developed by \emph{target alignment}  \citep{Cristianini2001On,Cortes2012Algorithms,Kumar2012A}. 
In stage 1, they consider finding a ``good" combination of base kernels using the training data based on \emph{target alignment} \citep{Cortes2012Algorithms,Wang2015An}.
Accordingly, the learned weight vector yields the learned kernel for the subsequent prediction process.
Stage 2 is a standard learning problem in RKHS $\mathcal{H}$ associated with the learned kernel.
Our framework for kernel learning is different from them in the hypothesis space.
Classical two-stage kernel learning framework in essence belongs to multiple kernel learning \citep{G2011Multiple} in RKHS due to the pre-given kernels. Nevertheless, our framework in hyper-RKHS does not restrict specific formulation on kernel.
Since a pre-given positive definite kernel can correspond to a fixed combination of pre-given elements in hyper-RKHS, the space spanned by a linear combination of PD kernels is only a small subspace in hyper-RKHS.
Hence, our kernel learning framework can be learned in this space from a broader class, which allows for significant model flexibility.
More importantly, the application of the studied framework is not limited to kernel learning. It can be also applied to out-of-sample extensions in non-parametric kernel learning to learn a underlying kernel/similarity function from a pre-given similarity matrix as demonstrated by Section~\ref{sec:onek}. This is actually beyond the topic of kernel learning, which in turn expands the application of learning in hyper-RKHS.

Actually, several representative approaches are able to achieve the similar effect as the used learning framework in hyper-RKHS for stage 1.
For example, learning by random features \citep{AmanNIPS2016} is able to work in a two-stage setting by first learning the weights of random features based on target alignment, and then obtaining a predictor. Such learning strategy in the spectral density sense is also used in \citep{bullins2017not} and can be further improved by generative models \citep{li2019implicit}. 
Besides, pairwise learning \citep{stock2018comparative,lei2020sharper} is an alternative way to achieve this target by constructing pairwise kernels, which measures the similarity between two pairs $(\bm x_1, \bm x_1')$ and $(\bm x_2, \bm x_2')$.
Such similarity learning on pair samples is also popular in deep metric learning, e.g., contrastive
embedding via Siamese networks \citep{bromley1994signature,guo2017learning}, triplet embedding \citep{salakhutdinov2007learning,hoffer2015deep} that jointly constitutes a positive pair and a negative pair.

It is worth nothing that kernel learning is not conflict with existing works in deep learning.
In fact, the connections between kernel methods and (deep) neural networks in over-parameterized setting have been extensively explored in recent years, e.g., the relations between Gaussian processes and infinitely wide multi-layer networks \citep{lee2017deep}; the equivalence between \emph{weakly/fully-trained} neural networks \citep{chizat2019lazy,ghorbani2019linearized} and kernel regression by random features \citep{rahimi2007random,mei2019generalization} or neural tangent kernel \citep{jacot2018neural} under some proper initialization; the equivalence between training a two-layer neural network via gradient descent and learning a data-adaptive kernel in a dynamic RKHS \citep{dou2020training}.
We remark upfront that connections to kernel methods is not the only way for analyzing (deep) neural networks.
The spanning space of neural networks is also not limited to RKHS.
For example, the ``dot-product attention" in Transformers can be characterized as kernel learning in Banach spaces \citep{wright2021transformers} instead of RKHS, which also leads to an indefinite kernel but not in RKKS. 
The functional space of two layer wide-width neural networks can be induced by the variation norm \citep{bach2017breaking,chizat2020implicit}, which is much larger than that of the RKHS norm for better understanding.
Further, many other approaches, with different points of views, have been proposed for deep learning theory, but they are out of scope of our discussion here.

\section{Conclusion}
\label{sec:con}

In this paper, we have studied the generalization properties of regularized regression models in hyper-RKHS.
The excess error converges at a certain learning rate as the sample size increases.
The derived learning rate provides a justification for us to learn the kernel in hyper-RKHS with theoretical guarantees.
Hence, we characterize a kernel learning framework in this space for kernel learning and out-of-sample extensions.
The studied framework in hyper-RKHS is quite general to cover a series of applications, e.g., kernel/metric learning, out-of-sample extensions.

\acks{
We thank the anonymous reviewers for their constructive and insightful comments.
The research leading to these results has received funding from the European Research Council under the European Union's Horizon 2020 research and innovation program / ERC Advanced Grant E-DUALITY (787960). This paper reflects only the authors' views and the Union is not liable for any use that may be made of the contained information.
This work was supported in part by Research Council KU Leuven: Optimization frameworks for deep kernel machines C14/18/068; Flemish Government: FWO projects: GOA4917N (Deep Restricted Kernel Machines: Methods and Foundations), PhD/Postdoc grant. This research received funding from the Flemish Government (AI Research Program). 
This work was supported in part by Ford KU Leuven Research Alliance Project KUL0076 (Stability analysis and performance improvement of deep reinforcement learning algorithms), EU H2020 ICT-48 Network TAILOR (Foundations of Trustworthy AI - Integrating Reasoning, Learning and Optimization), Leuven.AI Institute; and in part by the National Natural Science Foundation of China (Grants Nos. 61876107, 61977046, U1803261), National Key R\&D Program of China (No. 2019YFB1311503), and NSFC/RGC Joint Research Scheme (Nos. 1201101029 and N\_CityU102/20), in part by Shanghai Science and Technology Research Program (20JC1412700 and 19JC1420101), Shanghai Municipal Science and Technology Major Project (2021SHZDZX0102) and SJTU Global Strategic Partnership Fund (2020 SJTU-CORNELL).}

\appendices

\section{Proofs}
\label{sec:prooframe}
\subsection{Proofs of Proposition~\ref{errdec}}
\label{sec:errdec}
 \begin{proof}
According to the project operator $\pi_B$ in Definition~\ref{proj}, for any given
$a,b\in \mathbb{R}$, if $a\geq b$, we have
\begin{eqnarray*}
\pi_B(a)-\pi_B(b) = \left\{ \begin{array}{ll}
0 & \hbox{ if } a\geq b \geq B \hbox{ or } -B \geq a \geq b\,,\\
\min\{a,B\}+\min\{-b,B\} & \hbox{ otherwise }\,.\\
\end{array} \right.
\end{eqnarray*} Then we have $0 \leq \pi_B(a)-\pi_B(b)\leq a-b$ if
$a\geq b$. Similarly, when $a\leq b$, we have $a-b\leq \pi_B(a)-\pi_B(b)\leq 0$.
Hence for any $(\bm x, \bm x', y)\in Z$ and $k:X \times X \to \mathbb{R}$, there holds
\begin{equation*}
 \mathcal{T}^{\varepsilon}\big(\pi_B(y), \pi_B(k(\bm x, \bm x'))\big) \leq \mathcal{T}^{\varepsilon}\big(y, k(\bm x, \bm x')\big)\,.
\end{equation*}
Recall Eq.~\eqref{fzrs}, $\mathcal{E}_{\bm z}\big(\pi_B(k^{(\varepsilon)}_{\bm{z},\lambda})\big)$ can be bounded by
\begin{equation}\label{boundy}
\begin{split}
  \mathcal{E}_{\bm z}\big(\pi_B(k^{(\varepsilon)}_{\bm{z},\lambda})\big) &= \frac{1}{m^2} \sum_{i,j=1}^{m} \mathcal{T}^{\varepsilon}\big(y_{ij}, \pi_B(k(\bm x_i, \bm x_j))\big) \\
   &\leq  \frac{1}{m^2} \sum_{i,j=1}^{m} \mathcal{T}^{\varepsilon}\big(\pi_B(y_{ij}), \pi_B(k(\bm x_i, \bm x_j))\big) +  \frac{1}{m^2} \sum_{i,j=1}^{m} \Big| \pi_B(y_{ij}) - y_{ij} \Big| \\
   &\leq \mathcal{E}_{\bm z}\big(k^{(\varepsilon)}_{\bm{z},\lambda}\big) + \frac{1}{m^2} \sum_{i,j=1}^{m} \Big| \pi_B(y_{ij}) - y_{ij} \Big| \,,
  \end{split}
\end{equation}
where the second term $ \frac{1}{m^2} \sum_{i,j=1}^{m} | \pi_B(y_{ij}) - y_{ij} |$ is termed as the output error.
Accordingly, we have
\begin{equation*}
\begin{split}
&  \mathcal{E}\big(\pi_B(k^{(\varepsilon)}_{\bm{z},\lambda})\big) - \mathcal{E}(k_{\rho})  +  \lambda \| k^{(\varepsilon)}_{\bm{z},\lambda} \|^2_{\underline{\mathcal{H}}} \\
& = \bigg\{ \mathcal{E}\big(\pi_B(k^{(\varepsilon)}_{\bm{z},\lambda})\big) - \mathcal{E}_{\bm z}\big(\pi_B(k^{(\varepsilon)}_{\bm{z},\lambda})\big) - \mathcal{E}(k_{\rho}) \bigg\} +\mathcal{E}_{\bm z}\big(\pi_B(k^{(\varepsilon)}_{\bm{z},\lambda})\big)+  \lambda \| k^{(\varepsilon)}_{\bm{z},\lambda} \|^2_{\underline{\mathcal{H}}}   \\
& \leq \bigg\{ \mathcal{E}\big(\pi_B(k^{(\varepsilon)}_{\bm{z},\lambda})\big) - \mathcal{E}_{\bm z}\big(\pi_B(k^{(\varepsilon)}_{\bm{z},\lambda})\big) - \mathcal{E}(k_{\rho}) \bigg\} +\mathcal{E}_{\bm z}\big(k^{(\varepsilon)}_{\bm{z},\lambda}\big)+  \lambda \| k^{(\varepsilon)}_{\bm{z},\lambda} \|^2_{\underline{\mathcal{H}}}   + \frac{1}{m^2} \sum_{i,j=1}^{m} \Big| \pi_B(y_{ij}) - y_{ij} \Big| \\
& \leq \bigg\{ \mathcal{E}\big(\pi_B(k^{(\varepsilon)}_{\bm{z},\lambda})\big) - \mathcal{E}_{\bm z}\big(\pi_B(k^{(\varepsilon)}_{\bm{z},\lambda})\big) - \mathcal{E}(k_{\rho}) \bigg\} +\mathcal{E}_{\bm z}\big(k_{\lambda}\big)+ {\lambda \| k_{\lambda}\|^2_{\underline{\mathcal{H}}}}  + \varepsilon +\frac{1}{m^2} \sum_{i,j=1}^{m} \Big| \pi_B(y_{ij}) - y_{ij} \Big|\\
&:= D(\lambda)  + \mathcal{E}\big(\pi_B(k^{(\varepsilon)}_{\bm{z},\lambda})\big) - \mathcal{E}_{\bm{z}}\big(\pi_B(k^{(\varepsilon)}_{\bm{z},\lambda})\big) + \mathcal{E}_{\bm{z}}\big(k_{\lambda}\big) - \mathcal{E}\big(k_{\lambda}\big)+ \varepsilon + \frac{1}{m^2} \sum_{i,j=1}^{m} \Big| \pi_B(y_{ij}) - y_{ij} \Big|\,,
\end{split}
\end{equation*}
where the first inequality holds by Eq.~\eqref{boundy}, the second inequality satisfies because $k^{(\varepsilon)}_{\bm{z},\lambda}$ is the minimizer of Eq.~\eqref{fzrs} and the insensitivity condition in Eq.~\eqref{insentivec}, and the last equality admits by Eq.~\eqref{Dlamdadef}.
Finally, we draw our conclusion.
\end{proof}

\subsection{Proofs of Proposition~\ref{propos2}}
\label{sec:s2}
\begin{proof}
Considering the random variable $\xi$ in Eq.~\eqref{randomvariables} on $(Z,\rho)$, we have
\begin{equation*}
  S_2(\bm z, \lambda) = \frac{1}{m^2} \sum_{i,j=1}^{m} \xi(\bm z_{ij}) - \mathbb{E}(\xi) \leq \frac{1}{m(m-1)} \sum_{i,j=1}^{m} \sum_{j=1,j\neq i}^{m} \xi(\bm z_{ij}) + \frac{1}{m^2} \sum_{i=1}^{m} \xi(\bm z_{ii}) - \mathbb{E}(\xi)\,.
\end{equation*}
First, we consider the non-diagonal elements $\xi(\bm z_{ij})$ with $i \neq j$.
Since $\|k_{\lambda} \|_{\infty} \leq \mathcal{G} \sqrt{\frac{D(\lambda)}{\lambda}}$ by Eq.~\eqref{flambound} and $k_{\rho}(\bm x, \bm x')$ contained in $[-M^*,M^*]$.
Accordingly, we can get
\begin{equation*}
  \big|\xi - \mathbb{E}(\xi) \big| \leq   \mathcal{G} \sqrt{\frac{D(\lambda)}{\lambda}}+M^*\,.
\end{equation*}
By Lemma~\ref{lemmas3}, the variance-expectation condition of $\xi(\bm x, \bm x', y)$ is satisfied with $\theta$ given by Eq.~\eqref{theta} and
$c_1=C_{\theta}\max \{ B, M^*\}^{2-\theta}$.
Applying Lemma \ref{lemma4}, there exists a subset $Z_1$ of $Z^{m \times m}$ with confidence $1-\delta/4$, we have
\begin{equation}\label{bound6}
\frac{1}{m(m-1)}\sum_{i=1}^m \sum_{j\neq i}^{m}\xi(\bm z_{ij})-\mathbb{E}\xi \leq
\sqrt{(\mathbb{E}\xi)^{\theta}+\epsilon^{\theta}}
\epsilon^{1-\frac{\theta}{2}} \leq
(\mathbb{E}\xi)^{\frac{\theta}{2}}\epsilon^{1-\frac{\theta}{2}}+\epsilon
\leq \frac{\theta}{2}\mathbb{E}\xi +
\left(2-\frac{\theta}{2}\right)\epsilon,
\end{equation}
where the last inequality is from Young's inequality.
Let $\epsilon$ be the solution of the equation
\begin{equation*}
\exp\left\{-\frac{(m-1)\epsilon^{2-\theta}}{4C_{\theta}\max \{ B, M^*\}^{2-\theta}
+\frac{4}{3}(\mathcal{G} \sqrt{\frac{D(\lambda)}{\lambda}} + M^*)\epsilon^{1-\theta}}\right\}=\delta/4\,.
\end{equation*}
 Using Lemma 7.2 in \cite{cucker2007learning}, we find
\begin{equation*}
\begin{split}
\epsilon &\leq \max\left\{\frac{8\Big(\mathcal{G} \sqrt{\frac{D(\lambda)}{\lambda}} + M^*\Big) \log\frac{4}{\delta}}{3(m-1)},\left(\frac{8C_{\theta}
\max \{ B, M^*\}^{2-\theta}\log\frac{4}{\delta}}{m-1}\right)^{\frac{1}{2-\theta}}\right\} \\
& \leq 8(C_{\theta}+1)\bigg(\mathcal{G} \sqrt{\frac{D(\lambda)}{\lambda}} + M^* + B \bigg) \log\frac{4}{\delta}m^{-\frac{1}{2-\theta}} \,,
\end{split}
\end{equation*}
where we use $\frac{1}{m} \leq m^{-\frac{1}{2-\theta}}$ and $(\frac{1}{m-1})^{-\frac{1}{2-\theta}} \leq m^{-\frac{1}{2-\theta}}$ with $\theta \in (0,1]$ in Eq.~\eqref{theta}.
Substituting the above bound to Eq.~\eqref{bound6}, we obtain
\begin{equation*}
\begin{split}
\frac{1}{m(m-1)}\sum_{i=1}^m \sum_{j\neq i}^{m}\xi(\bm z_{ij})-\mathbb{E}\xi
 &\leq
\frac{\theta}{2}\! \Big\{ \mathcal{E}(k_{\lambda}) - \mathcal{E}(k_{\rho}) \Big\} \!
\!+\! 8\! \left(2-\frac{\theta}{2}\right)\! (C_{\theta}\!+\!1)\bigg(\mathcal{G} \sqrt{\frac{D(\lambda)}{\lambda}} \!+\! M^*\!+\!B\bigg) \log\frac{4}{\delta}m^{-\frac{1}{2-\theta}}\\
&\leq \frac{1}{2}{D}(\lambda)
+16(C_{\theta}+1)\bigg(\mathcal{G} \sqrt{\frac{D(\lambda)}{\lambda}} + M^*+ B\bigg) \log\frac{4}{\delta}m^{-\frac{1}{2-\theta}}\,.
\end{split}
\end{equation*}

Next we consider the diagonal elements $\xi(\bm z_{ij})$ with $i = j$, that is
\begin{equation*}
  \frac{1}{m^2} \sum_{i=1}^{m}  {\xi}(\bm x_i, \bm x_i, y_{ii}) \leq \frac{1}{m} \Big( \mathcal{G} \sqrt{\frac{D(\lambda)}{\lambda}} + M^* \Big)\,.
\end{equation*}
Finally, combining above two equations, we have
\begin{equation*}
  S_2({\bm z, \lambda}) \leq \frac{1}{2}{D}(\lambda)
+16(C_{\theta}+1)\bigg(\mathcal{G} \sqrt{\frac{D(\lambda)}{\lambda}} + M^*+B\bigg) \log\frac{4}{\delta}m^{-\frac{1}{2-\theta}} + \frac{1}{m} \Big( \mathcal{G} \sqrt{\frac{D(\lambda)}{\lambda}} + M^* \Big)\,,
\end{equation*}
which concludes the proof.
\end{proof}

\subsection{Proofs of Proposition~\ref{propos1}}
\label{sec:s1}
\begin{proof}
Consider the function set $\mathcal{F}_R$ with $R>0$ by
\begin{equation*}
    \mathcal{F}_R := \left\{ \mathcal{T}\big(y,\pi_B(k)(\bm x, \bm x')\big) - \mathcal{T}\big(y,k_{\rho}(\bm x, \bm x')\big) : k \in \mathcal{B}_R \right\}\,.
\end{equation*}
Each function $g \in \mathcal{F}_R$ has the form $g(\bm x, \bm x',y) = \mathcal{T}\big(y,\pi_B(k)(\bm x, \bm x')\big) - \mathcal{T}\big(y,k_{\rho}(\bm x, \bm x')\big)$ with some $k \in \mathcal{B}_R$.
Hence, $S_1$ can be bounded by
\begin{equation}\label{S2dec}
\begin{split}
 S_1(\bm z, \lambda) & \leq \left( \frac{1}{m(m-1)} \sum_{i=1}^{m} \sum_{j \neq i}^{m} g(\bm x_i, \bm x_j, y_{ij}) - \mathbb{E}g +  \frac{1}{m^2} \sum_{i=1}^{m}  g(\bm x_i, \bm x_i, y_{ii}) \right) \,.
 \end{split}
\end{equation}
We can easily see that $\| g\|_{\infty} \leq B+M^*$, and thus we have $| g-\mathbb{E}g| \leq B+M^*$.
By Lemma~\ref{lemmas3}, the variance-expectation condition of $\xi(z)$ is satisfied with $\theta$ given by Eq.~\eqref{theta} and
$c_1=C_{\theta}\max\{B,M^*\}^{2-\theta}$.

First, we consider the $i \neq j$ case by Lemma~\ref{lemma4}.
The Lipschitz property of the $\varepsilon$-insensitive loss yields $\mathscr{N}(\mathcal{F}_R,\epsilon) \leq \mathscr{N}(\mathcal{B}_1,\epsilon)$.
So applying Lemma~\ref{lemma4} to the function set $\mathcal{F}_R$ with the covering number condition in Eq.~\eqref{assumpN}, we have
  \begin{equation*}
  \begin{split}
  &\mathop{\mathrm{Prob}} \limits_{\bm z \in Z^{m \times m}} \Bigg\{ \sup_{k \in \mathcal{F}_R}  \frac{\mathbb{E}g - \frac{1}{m(m-1)} \sum_{i=1}^{m} \sum_{j \neq i}^{m} g(\bm x_i, \bm x_j, y_{ij})  }{\sqrt{(\mathbb{E}g)^{\theta} +\epsilon^{\theta}}} \geq 4\epsilon^{1-\frac{\theta}{2}} \Bigg\} \\
  &\qquad \qquad \leq \mathscr{N}(\mathcal{B}_1,\epsilon) \exp \bigg\{ -\frac{(m-1) \epsilon^{2-\theta}}{4C_{\theta}\max\{B,M^*\}^{2-\theta} +\frac{16}{3}M^* \epsilon^{1-\theta}} \bigg\} \\
  &\qquad \qquad \leq \exp \left\{ C_s \Big( \frac{R}{\epsilon} \Big)^s -\frac{(m-1) \epsilon^{2-\theta}}{4C_{\theta}\max\{B,M^*\}^{2-\theta} +\frac{4}{3}(B+M^*) \epsilon^{1-\theta}} \right\}\,,
  \end{split}
\end{equation*}
where $\mathbb{E}g = \mathcal{E}\big(\pi_B(k)\big) - \mathcal{E}(k_{\rho})$.
Hence there holds a subset $Z_2$ of $Z^{m \times m}$ with confidence at least $1-\delta/4$ such that
\begin{equation*}
  \sup_{k \in \mathcal{F}_R}  \frac{\mathbb{E}g - \frac{1}{m(m-1)} \sum_{i=1}^{m} \sum_{j \neq i}^{m} g(\bm x_i, \bm x_j, y_{ij})  }{\sqrt{(\mathbb{E}g)^{\theta} +\Big(\epsilon^*(m,R,\frac{\delta}{4})\Big)^{\theta}}} \leq  4\big(\epsilon^*(m,R,\frac{\delta}{4})\big)^{1-\frac{\theta}{2}} \quad \forall \bm z \in Z_2 \cap \mathscr{W}(R) \,,
\end{equation*}
where $\epsilon^*(m,R,\frac{\delta}{4})$ is the smallest positive number $\epsilon$ satisfying
\begin{equation*}
  C_s \Big( \frac{R}{\epsilon} \Big)^s -\frac{m \epsilon^{2-\theta}}{4C_{\theta}\max\{B,M^*\}^{2-\theta} +\frac{4}{3}(M^*+B) \epsilon^{1-\theta}} = \log \frac{\delta}{4}\,,
\end{equation*}
using Lemma 7.2 in \cite{cucker2007learning}, we have
\begin{equation*}
\begin{split}
\epsilon^* & \leq \max\Bigg\{\frac{16(M^*+B)}{3m}\log\frac{4}{\delta},\left(\frac{16C_{\theta}
\max\{B,M^*\}^{2-\theta}}{m}\log\frac{4}{\delta}\right)^{\frac{1}{2-\theta}}, \left( \frac{16(M^*+B)}{3m}C_sR^s \right)^{\frac{1}{1+s}}, \\
&~\quad~~\quad~~\quad~ \left( \frac{16\max\{B,M^*\}^{2-\theta}}{3m}C_sR^sC_{\theta} \right)^{\frac{1}{2-\theta+s}} \Bigg\} \\
& \leq 16(M^*+B)(C_{\theta}+1)\log\frac{4}{\delta} m^{-\frac{1}{2-\theta}} + 16C_s(M^*+B)(C_{\theta}+1)m^{-\frac{1}{2+s-\theta}}R^{\frac{s}{1+s}}\,,
\end{split}
\end{equation*}
 where we use $\frac{1}{m} \leq m^{-\frac{1}{2-\theta}}$, $M^*\geq 1$ and $(M^*)^{\frac{1}{1+s}} \leq (M^*)^{\frac{2-\theta}{2-\theta+s}}$.

 Next we consider the $i = j$ case in $S_1(\bm z, \lambda)$.
 Since $\| g\|_{\infty} \leq B+M^*$, we have
\begin{equation*}
  \frac{1}{m^2} \sum_{i=1}^{m}  g(\bm x_i, \bm x_i, y_{ii})  \leq \frac{1}{m} \Big( B+M^* \Big)\,.
\end{equation*}
Combining above two equations, for $\bm z \in \mathcal{B}(R) \cap Z_2$, we have
 \begin{equation*}
 \begin{split}
 S_1(\bm z, \lambda) &= \Big\{ \mathcal{E}\big(\pi_B(k^{(\varepsilon)}_{\bm{z},\lambda})\big)\! - \! \mathcal{E}(k_{\rho}) \Big\} \!-\! \Big\{ \mathcal{E}_{\bm{z}}\big(\pi_B(k^{(\varepsilon)}_{\bm{z},\lambda})\big) \!- \! \mathcal{E}_{\bm{z}}(k_{\rho}) \Big\}  \\
 & \leq 4\big[\epsilon^*(m,R,\frac{\delta}{4})\big]^{1-\frac{\theta}{2}} \sqrt{\Big( \mathcal{E}\big(\pi(k_{\bm{z},\lambda})\big) - \mathcal{E}(k_{\rho}) \Big) ^{\theta} +\Big(\epsilon^*(m,R,\frac{\delta}{4})\Big)^{\theta}} + \frac{1}{m}(B+M^*) \\
 & \leq (1-\frac{\theta}{2})4^{2/(2-\theta)}\epsilon^*(m,R,\frac{\delta}{4}) + \frac{\theta}{2} \Big( \mathcal{E}\big(\pi_B(k^{(\varepsilon)}_{\bm{z},\lambda})\big) - \mathcal{E}(k_{\rho}) \Big) + 4\epsilon^*(m,R,\frac{\delta}{4}) + \frac{1}{m}(B+M^*) \\
 & \leq 20 \epsilon^*(m,R,\frac{\delta}{4}) +  \frac{1}{2}\Big\{ \mathcal{E}\big(\pi_B(k^{(\varepsilon)}_{\bm{z},\lambda})\big) - \mathcal{E}(k_{\rho}) \Big\} + \frac{1}{m}(B+M^*) \,,
 \end{split}
 \end{equation*}
 where the second inequality  is from Young's inequality.
Finally, we complete the proof.
 \end{proof}

\subsection{Proofs of Proposition~\ref{proexcessfin}}
\label{sec:excessfin}
\begin{proof}
Combining the bounds in Proposition~\ref{errdec},~\ref{propos2},~\ref{propos1}, Eq.~\eqref{flambound} and Eq.~\eqref{boundyy}, the excess error $ \mathcal{E}\big(\pi_B(k^{(\varepsilon)}_{\bm{z},\lambda})\big) - \mathcal{E}(k_{\rho})$ can be bounded by
\begin{equation}\label{finalb}
\begin{split}
  \mathcal{E}\big(\pi(k^{(\varepsilon)}_{\bm{z},\lambda})\big) - \mathcal{E}(k_{\rho}) & \leq 2 \varepsilon + 3C_0\lambda^{r} + 32(C_{\theta}+1)\Big(M^*+\mathcal{G} \sqrt{C_0}\lambda^{\frac{r-1}{2}}\Big) \log\frac{4}{\delta}m^{-\frac{1}{2-\theta}} + \frac{2\mathcal{G} \sqrt{C_0}\lambda^{\frac{r-1}{2}}}{m} \\
  & + 640 (C_{\theta}+1)(M^*+B)\bigg( \log\frac{4}{\delta} m^{-\frac{1}{2-\theta}} + C_sm^{-\frac{1}{2+s-\theta}}R^{\frac{s}{1+s}} \bigg) \\
  & + \frac{2B+4M^*}{m} + 2c \Big\{ (d+1)! + 2^{d+2} d^d \big\} M^{d+1} B^{-d} + \frac{12M 2^{d+1}}{m} \log \frac{4}{\delta}\,.
    \end{split}
\end{equation}

In the next, we attempt to find a $R>0$ by giving a bound for $\lambda \| k^{(\varepsilon)}_{\bm{z},\lambda} \|^2_{\underline{\mathcal{H}}}$. Form the definition of $k^{(\varepsilon)}_{\bm{z},\lambda}$ in Eq.~\eqref{fzrs}, we have
\begin{equation*}
  \lambda \| k^{(\varepsilon)}_{\bm{z},\lambda} \|^2_{\underline{\mathcal{H}}} \leq \mathcal{E}_{\bm{z}}\big(k^{(\varepsilon)}_{\bm{z},\lambda}\big) +  \lambda \| k^{(\varepsilon)}_{\bm{z},\lambda} \|^2_{\underline{\mathcal{H}}} \leq \mathcal{E}_{\bm{z}}(0) \leq \frac{1}{m^2} \sum_{i,j=1}^{m} |Y_{ij}|\,.
\end{equation*}
Using Eq.~\eqref{boundyy} with confidence
$1-\delta/4$, we have
\begin{equation}\label{bound16}
 \frac{1}{m^2} \sum_{i,j=1}^{m} |Y_{ij}| \leq cM+4M(1+\sqrt{2c})\frac{\log
\frac{4}{\delta}}{\sqrt{m}} \leq (3cM+4M)\log
\frac{4}{\delta}:=M_{\delta}.
\end{equation} This yields the measure
of the set $\mathscr{W}(\frac{M_{\delta}}{\lambda})$ is at least
$1-\delta/4$, thus the measure of the set
$\mathscr{W}(\frac{M_{\delta}}{\lambda})\cap Z_3 \cap Z_2 \cap Z_1$
is at least $1-\delta$. We substitute $R=\frac{M_{\delta}}{\lambda}$
to Eq.~\eqref{finalb} and let Eq.~\eqref{assumpN} with $s>0$, Eq.~\eqref{Dlambda} with $0 < r \leq 1$, take $\lambda=m^{-\alpha}$ with $ 0 < \alpha \leq 1$ and $\alpha < \frac{1+s}{s(2+s-\theta)}$. Set $\varepsilon = m^{-\gamma}$ with $\alpha r \leq \gamma \leq \infty$, we have
\begin{equation*}
\begin{split}
  \mathcal{E}\big(\pi(k^{(\varepsilon)}_{\bm{z},\lambda})\big) - \mathcal{E}(k_{\rho}) & \leq 2 m^{-\gamma} + 3C_0m^{-\alpha r} +\widetilde{C}_1\log\frac{4}{\delta}m^{-\frac{1}{2-\theta}}
  + \widetilde{C}_2 \log\frac{4}{\delta}m^{-\big[ \frac{1}{2-\theta}
  + \frac{\alpha(r-1)}{2-\theta} \big]}\\
   &+ 2\mathcal{G} \sqrt{C_0}m^{-\big[\frac{\alpha(r-1)}{2}+1\big]}
   + \widetilde{C}_3(3cM+4M)m^{\frac{\alpha s}{1+s}-\frac{1}{2+s-\theta}}\\
  & + \frac{2B+4M^*}{m} + 2c \Big\{ (d+1)! + 2^{d+2} d^d \big\} M^{d+1} B^{-d} + \frac{12M 2^{d+1}}{m} \log \frac{4}{\delta}\,,
    \end{split}
\end{equation*}
where $\widetilde{C}_1$, $\widetilde{C}_2$ and $\widetilde{C}_3$ are constants given by
\begin{eqnarray*}
\widetilde{C}_1 = (C_{\theta}+1)(672M^*+640B),~~~
  \widetilde{C}_2 =  32(C_{\theta}+1)\mathcal{G} \sqrt{C_0},~~~
  \widetilde{C}_3 = 640 (C_{\theta}+1)(M^*+B) C_s\,.
\end{eqnarray*}

Formally, we choose
\begin{equation*}
  \widetilde{C}  = \max \Big\{3C_0, \widetilde{C}_1,  \widetilde{C}_2, 2\mathcal{G} \sqrt{C_0}, \widetilde{C}_3(3cM+4M), 2B+4M^*\Big \} =  \max \Big\{3C_0, \widetilde{C}_1,  \widetilde{C}_2, 2\mathcal{G} \sqrt{C_0}, \widetilde{C}_3(3cM+4M)\Big \} \,,
\end{equation*}
and
the power index $\Theta$
\begin{equation*}
\begin{split}
  \Theta &= \min \left\{ \gamma, {\alpha r}, \frac{1}{2-\theta}, \frac{1}{2-\theta}
  + \frac{\alpha(r-1)}{2-\theta}, \frac{\alpha(r-1)}{2}+1, \frac{1}{2+s-\theta} - \frac{\alpha s}{1+s} \right\} \\
 & = \min \left\{ {\alpha r},   \frac{1+\alpha r -\alpha}{2-\theta}
 , \frac{1}{2+s-\theta} - \frac{\alpha s}{1+s}  \right\}\,,
  \end{split}
\end{equation*}
which concludes the proof.
\end{proof}

\newpage


\end{document}